%% file: main.tex
\documentclass{article}

\usepackage[preprint]{neurips_2026}

\usepackage[utf8]{inputenc} 
\usepackage[T1]{fontenc}    
\usepackage{hyperref}       
\usepackage{url}            
\usepackage{booktabs}       
\usepackage{wrapfig}
\usepackage{caption}
\usepackage{amsthm}
\usepackage{enumitem}
\usepackage{graphicx}
\usepackage{colortbl}
\usepackage{amsmath}
\usepackage{amssymb}
\usepackage{algorithmic}
\usepackage{algorithm}
\usepackage{amsfonts}       
\usepackage{nicefrac}       
\usepackage{microtype}      
\usepackage{xcolor}         
\usepackage{colortbl}
\usepackage{wrapfig}
\usepackage{makecell}

\definecolor{pos5}{HTML}{E8F5E9}
\definecolor{pos10}{HTML}{C8E6C9}
\definecolor{pos30}{HTML}{A5D6A7}
\definecolor{pos60}{HTML}{81C784}
\definecolor{pos100}{HTML}{66BB6A}
\definecolor{neg5}{HTML}{FFEBEE}
\definecolor{neg15}{HTML}{FFCDD2}
\definecolor{neg20}{HTML}{EF9A9A}
\definecolor{neg37}{HTML}{E57373}
\definecolor{basegray}{HTML}{F5F5F5}

\title{Reward Transport: Property Control in Flow Matching
    via Noise-Space Alignment}

\newtheorem{proposition}{Proposition}

\author{%
  Kehan Guo$^{1}$ \quad Yili Shen$^{1}$ \quad Yujun Zhou$^{1}$ \quad Yue Huang$^{1}$ \\[3pt]
  \textbf{Chujie Gao}$^{1}$ \quad \textbf{Shiyi Du}$^{2}$ \quad \textbf{Xiangliang Zhang}$^{1}$ \\[5pt]
  $^{1}$University of Notre Dame \qquad $^{2}$Carnegie Mellon University \\[2pt]
  \texttt{kguo2@nd.edu} \quad \texttt{xzhang33@nd.edu}
}

\begin{document}

\maketitle

\begin{abstract}
The coupling in flow matching---the rule pairing noise vectors with data points---is typically treated as a computational choice. 
We show that this coupling can instead serve as an alignment interface: by matching noise and data according to a target molecular property, it embeds controllable structure directly into the learned flow field. Building on this view, we introduce Reward Transport, which uses optimal transport coupling at training time to align a scalar noise-space coordinate with molecular rewards; at inference, varying this coordinate steers the generated distribution without requiring an oracle, reward model, gradient guidance, or additional computation. 
In the coupling-preserving limit, thresholding this coordinate recovers the Cross-Entropy Method’s truncated reward distribution, providing a principled, continuously adjustable distribution-level control knob. 
Empirically, on ZINC-250K and GuacaMol, sweeping the scalar induces monotone control of logP and consistent QED control over its operating range; most tellingly, the same knob produces opposite structural responses for different targets, growing molecules for logP but shrinking them for QED, which rules out a generic size bias.
The interface is complementary to classifier-free guidance and conditional flow matching, while a negative result under $\epsilon$-prediction diffusion clarifies where coupling-level alignment is structurally absent.
Our code is available at: \href{https://github.com/KehanGuo2/reward-transport}{https://github.com/KehanGuo2/reward-transport}.
\end{abstract}

\input{1-intro}
\input{2-related_work}
\input{3-Method}

\input{4-experiment}
\input{6-discussion}
\input{7-conclusion}



\newpage
\bibliographystyle{plainnat}
\bibliography{references}
\newpage
\input{5-appendix}


\end{document}

%% file: 1-intro.tex
\section{Introduction}
\label{sec:introduction}

\looseness=-1
A flow matching model depends on both the vector-field training objective and the \emph{coupling}: the rule that determines which noise vector is paired with which data point during training. While prior work has carefully studied training objectives ~\citep{karras2022elucidating, esser2024scaling, ma2024sit}, the coupling is often treated as a background implementation choice:
independent pairing for simplicity~\citep{lipman2022flow}, or minibatch optimal transport to straighten trajectories and reduce path crossings~\citep{Tong2023ImprovingAG, pooladian2023multisample}.
In every case, the coupling is chosen to make training easier, not to shape what the model ultimately learns.

We take a different view. The coupling is not a training detail but an \emph{alignment interface}:
it determines which noise samples are associated with which data samples, and therefore whether meaningful data properties become organized in noise space. In most controllable generative models~\citep{ho2022classifier, dhariwal2021diffusion}, steerability is introduced through explicit inputs, such as class labels, property embeddings, or classifier-free guidance. In flow matching, we show that steerability can also be introduced through the coupling itself.
By choosing this assignment according to a target property, the learned flow field inherits a property-aligned organization before any inference-time guidance is applied~\citep{peebles2023dit, zeng2025propmolflow}.


\looseness=-1
Molecules are a natural testbed. They have scalar targets that matter—logP~\citep{wildman1999prediction} and QED~\citep{bickerton2012quantifying}—discrete variable-length samples that stress naive couplings, and properties cheap to verify without an oracle.
We study coupling as an alignment interface in the molecular regime and ask whether choosing it alone is enough to buy controllable generation, using property labels only at training time to build the coupling.

\looseness=-1
We instantiate the interface as \emph{Reward Transport} (Figure~\ref{fig:overview}): a property-aligned monotone coupling that sorts noise vectors by a scalar coordinate $s$ and molecules by a target property $y$, then pairs them rank-by-rank~\citep{villani2003topics}.
The theoretical hinge is simple: under this 1D rearrangement, choosing $s$ at inference---in the coupling-preserving limit---induces the same truncated distribution as one selection step of the Cross-Entropy Method~\citep{rubinstein1999cross}, with no oracle, no gradient, and no rejection (Proposition~\ref{prop:cem}).
A single scalar knob, set once before generation, becomes a distributional statement.

\looseness=-1
Empirically on ZINC-250K~\citep{irwin2012zinc}, that single knob steers logP monotonically over a $3.14$-unit range ($+137\%$, $\rho(s, \bar y){=}1.000$) at $100\%$ validity and zero inference overhead; the same knob, trained on QED, moves QED in the same monotone way (Table~\ref{tab:main}), and the effect reproduces on GuacaMol~\citep{brown2019guacamol}.
More importantly, the knob does not encode a generic structural bias.
The same coordinate \emph{grows} logP molecules from $12$ to $23$ heavy atoms and \emph{shrinks} QED molecules from $23$ to $16$. 
Thus, $s$ is not simply a size coordinate. It represents whichever property the coupling aligns with noise space: the coupling writes the task structure, and the flow transports it.

Our contributions in this work are summarized as follows:
\begin{itemize}[leftmargin=1.5em, itemsep=2pt, topsep=2pt]
    \item \textbf{Coupling as alignment interface.} We reframe the noise–data coupling in flow matching as  a mechanism for injecting task structure into the learned flow field, rather than a training-time heuristic.  In the coupling-preserving limit, property-aligned monotone coupling induces the same truncated distribution as one Cross-Entropy Method selection step (Proposition~\ref{prop:cem}), giving distributional control at inference through a single noise scalar with no added inference cost.
    \item \textbf{Single-knob molecular steering} 
    We demonstrate distribution-level steering of logP and QED on ZINC-250K~\citep{irwin2012zinc} and GuacaMol~\citep{brown2019guacamol} (Table~\ref{tab:main}). The key insight is that the same knob produces opposite structural responses for different targets, ruling out a generic size bias and showing that the coupling writes target-specific structure into the flow.
 \item \textbf{Reward Transport for SELFIES generation.}
We implement Reward Transport for variable-length SELFIES generation, where a continuous flow generates discrete molecular token sequences. We identify and address two failure modes: magnitude erasure under Pre-LayerNorm~\citep{xiong2020layer} and a length-inflation shortcut under padded MSE. This generalizes to other flow models generating discrete, variable-length sequences.
\end{itemize}

%% file: 2-related_work.tex
\section{Related Work}
\label{sec:related_work}
\paragraph{Couplings in flow matching.}
Flow matching~\citep{lipman2022flow, liu2023rectified, albergo2023stochastic} learns an ODE-defined transport from noise to data, in the broader diffusion lineage~\citep{sohldickstein2015deep, ho2020denoising, song2019generative, karras2022elucidating, chen2018neural}. The transport depends on a \emph{coupling}: the joint distribution over noise and data used during training. Standard flow matching draws noise and data independently. OT-CFM~\citep{Tong2023ImprovingAG} and multisample flow matching~\citep{pooladian2023multisample} replace the independent coupling with a minibatch optimal-transport coupling, using Euclidean cost on noise–data pairs to straighten trajectories and reduce training variance. The Schrödinger-bridge view~\citep{debortoli2021diffusion} places the family inside entropic OT~\citep{cuturi2013sinkhorn, peyre2019computational, villani2003topics}. In all of these, the coupling is tuned for optimization; the OT cost is a distance. Reward Transport uses the same coupling machinery with a property-valued cost: noise and data are sorted by a scalar property and paired by rank. The coupling still solves a 1D optimal-transport problem, but the solution now encodes controllable property structure rather than shorter trajectories.

\looseness=-1
\paragraph{Property control in molecular generation.}
Molecular generators use many representations---autoregressive SMILES~\citep{weininger1988smiles, segler2018generating}, junction-tree VAEs~\citep{jin2018junction}, graph normalizing flows~\citep{shi2020graphaf, zang2020moflow}, graph diffusion~\citep{vignac2022digress}, 3D equivariant diffusion~\citep{hoogeboom2022equivariant, jing2022torsional}---benchmarked on GuacaMol~\citep{brown2019guacamol}. We use SELFIES~\citep{krenn2020selfies} tokens for syntactic validity. Methods that control generated properties differ in how property information reaches the sampler. Reinforcement learning fine-tunes the generator against a property oracle~\citep{blaschke2020reinvent, you2018graph, zhou2019moldqn, bengio2023gflownet}; the oracle is queried during training and often during fine-tuning-time sampling. Guidance methods evaluate a property predictor at every sampling step and use its gradient to steer the trajectory~\citep{ho2022classifier, lee2023molguidance}. Conditional flow matching takes the property value as a network input at training and inference time~\citep{zeng2025propmolflow,hou2024improving, deng2026generative}. All three require property information at sampling time—an oracle, a gradient, or a target value. Reward Transport requires none: the property is used only to build the training coupling, and sampling is driven by a scalar noise coordinate. Concurrent work on Bayesian flows over molecular graphs~\citep{xiong2025hierarchical} also modifies the coupling, using a quasi-Wasserstein geometric cost together with sampling-time conditioning, which is complementary to our design.

%% file: 3-Method.tex

\section{Background}
\label{sec:background}

\begin{figure}[t]
    \centering
    \includegraphics[width=\linewidth,trim=0 0 0 12,clip]{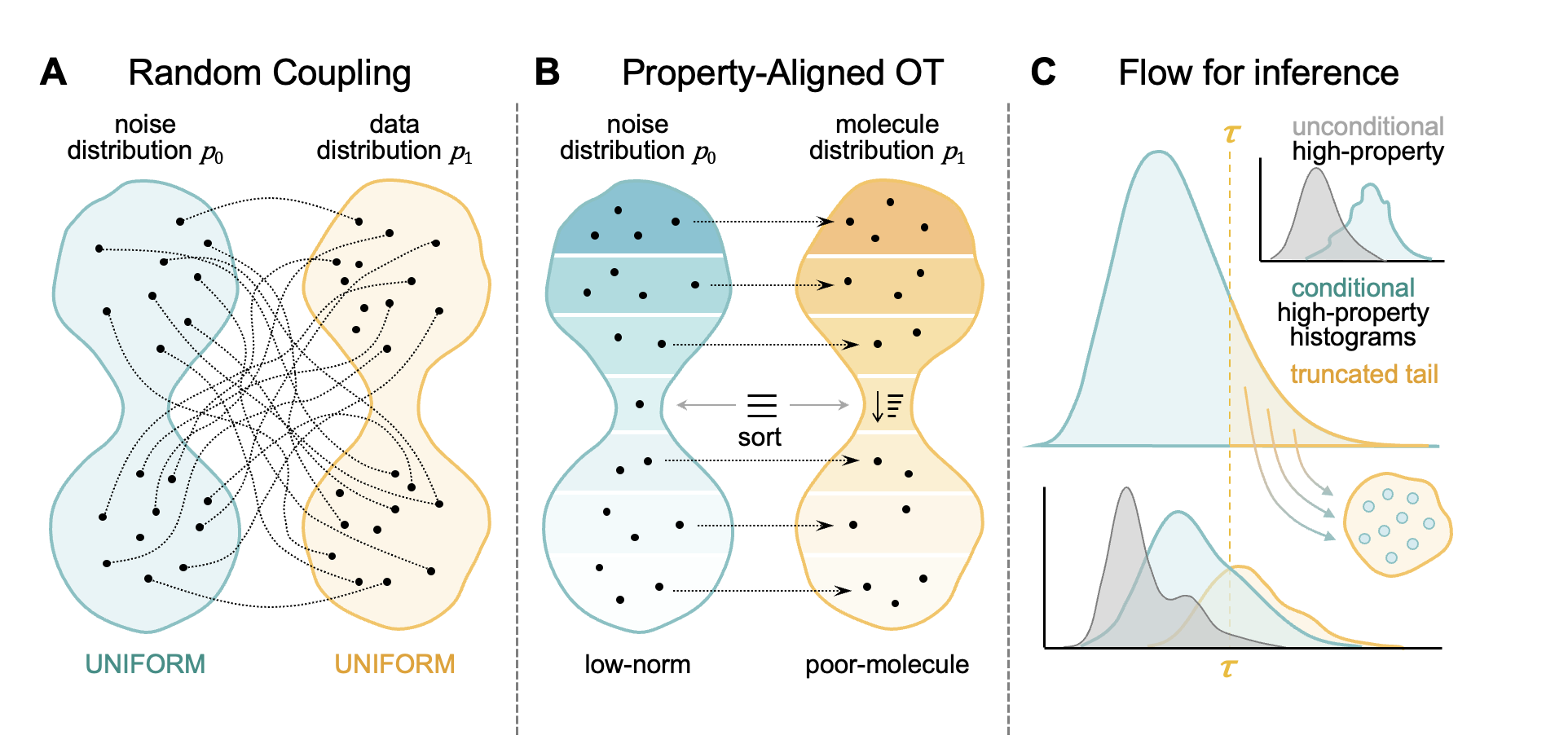}
    \vspace{-0.15in}
    \caption{\textbf{Reward Transport overview.}
    \textbf{(A)}~Standard flow matching pairs noise and data independently, producing no property structure in noise space.
    \textbf{(B)}~Property-aligned OT coupling sorts noise vectors by norm and molecules by target property, then pairs them rank-by-rank, creating a monotone mapping between noise coordinates and molecular properties.
    \textbf{(C)}~At inference, sampling from the upper tail of the noise distribution ($s \geq \tau$) yields a truncated property distribution equivalent to one step of the Cross-Entropy Method (Proposition~\ref{prop:cem}), shifting generated molecules toward higher property values---without any oracle, reward model, or additional computation.}
    \label{fig:overview}
\end{figure}

\paragraph{Flow Matching.}
Flow Matching~\citep{lipman2022flow} learns a velocity field $\mathbf{v}_\theta$ that transforms a prior $p_0 = \mathcal{N}(\mathbf{0}, \mathbf{I})$ into a data distribution $p_1$.
Given a noise--data pair $(\mathbf{z}, \mathbf{x})$ drawn from a coupling $\pi(\mathbf{z}, \mathbf{x})$, the linear interpolation path is
\begin{equation}
\label{eq:interpolation}
    \mathbf{x}_t = (1 - t)\,\mathbf{z} + t\,\mathbf{x}, \qquad t \in [0,1].
\end{equation}
The network $\mathbf{v}_\theta$ is trained to regress the conditional velocity $\mathbf{u} = \mathbf{x} - \mathbf{z}$:
\begin{equation}
\label{eq:fm_loss}
    \mathcal{L}_{\mathrm{FM}}
    = \mathbb{E}_{t,\,\mathbf{z},\,\mathbf{x}}
      \bigl\lVert \mathbf{v}_\theta(\mathbf{x}_t, t) - (\mathbf{x} - \mathbf{z}) \bigr\rVert^2.
\end{equation}
Samples are generated by integrating $\mathrm{d}\mathbf{x}_t/\mathrm{d}t = \mathbf{v}_\theta(\mathbf{x}_t, t)$ from $t{=}0$ to $t{=}1$; we write $\mathbf{x}_\theta(\mathbf{z})$ for the resulting sample.
The coupling $\pi$ determines which noise vector flows to which data point; standard FM draws pairs independently ($\pi = p_0 \otimes p_1$).

\section{Reward Transport}
\label{sec:method}

\subsection{Property-Aligned Optimal Transport Coupling}
\label{sec:ot_coupling}

We represent each molecule as a sequence of $L$ token embeddings in $\mathbb{R}^d$, so noise and data both lie in $\mathbb{R}^{L \times d}$.
The coupling $\pi(\mathbf{z}, \mathbf{x})$ pairs each noise sequence with a data sequence; we replace the standard random pairing with a \emph{monotone} coupling that aligns a scalar summary of the noise with a target molecular property $y(\mathbf{x}) \in \mathbb{R}$.
Let $\bar{\mathbf{z}} = \frac{1}{L}\sum_{i=1}^{L} \mathbf{z}_i \in \mathbb{R}^d$ denote the mean-pooled noise and define the sorting key\footnote{Alternative scalar keys ($\|\cdot\|_\infty$, random linear projection, or a single coordinate) yield comparable results; see Appendix~\ref{app:key_ablation}.}
\begin{equation}
\label{eq:sorting_key}
    s(\mathbf{z}) = \lVert \bar{\mathbf{z}} \rVert_2.
\end{equation}
Given a training set $\{(\mathbf{x}_n, y_n)\}_{n=1}^{N}$ and a fresh sample of noise $\{\mathbf{z}_n\}_{n=1}^{N}$, we sort both sequences by their respective keys and pair rank-by-rank:
\begin{equation}
\label{eq:monotone_coupling}
    \pi^* \colon \mathbf{z}_{\sigma(k)} \;\leftrightarrow\; \mathbf{x}_{\tau(k)},
    \qquad \sigma = \operatorname{argsort}(s(\mathbf{z}_{1:N})),
    \quad \tau = \operatorname{argsort}(y(\mathbf{x}_{1:N})).
\end{equation}
Since $s$ and $y$ are one-dimensional, this is the unique optimal transport plan under any convex cost---the classical monotone rearrangement~\citep{villani2003topics}.

We now formalize what the coupling buys at inference.

\begin{proposition}[Property-Aligned OT as Implicit CEM]
\label{prop:cem}
Let $\pi^*$ be the monotone coupling between $s(\mathbf{z})$ and $y(\mathbf{x})$, with CDFs $F_S, F_Y$, and let
\[
\rho_{\text{per}}(s, y) \;=\; \mathrm{corr}_{\text{Spearman}}\!\bigl(s(\mathbf{z}),\; y(\mathbf{x}_\theta(\mathbf{z}))\bigr)
\]
denote the per-molecule rank correlation between the noise-space key and the generated property.
In the coupling-preserving limit ($\rho_{\text{per}} \to 1$), for any threshold $\tau$:
\begin{equation}
\label{eq:cem}
    p_\theta\!\left(\mathbf{x} \mid s(\mathbf{z}) \geq \tau\right)
    \;\approx\;
    \frac{p_1(\mathbf{x})\;\mathbf{1}\!\bigl[y(\mathbf{x}) \geq y_\tau\bigr]}{Z(\tau)},
    \qquad
    y_\tau = F_Y^{-1}\!\bigl(F_S(\tau)\bigr),
\end{equation}
where $Z(\tau) = 1 - F_Y(y_\tau)$.
\end{proposition}

\textit{Proof sketch.}\;
By the monotone coupling, the $\alpha$-quantile of $s$ maps to the $\alpha$-quantile of $y$.
Conditioning on $s \geq \tau$ selects the upper $\varepsilon$-fraction of the noise distribution, where $\varepsilon = 1 - F_S(\tau)$; monotonicity guarantees the paired molecules lie in their upper $\varepsilon$-fraction, yielding Eq.~\eqref{eq:cem}. \hfill$\square$

\looseness=-1
Equation~\eqref{eq:cem} recovers the functional form of one Cross-Entropy Method~\citep{rubinstein1999cross} selection step with elite fraction $\varepsilon = 1 - F_S(\tau)$, obtained in a single forward pass: if the learned flow preserves the rank correspondence ($\rho_{\text{per}}\to 1$), the conditional distribution \emph{is} the data distribution truncated to $\{y \geq y_\tau\}$; the realised $\rho_{\text{per}}$ (§\ref{sec:results}: $0.57$ for logP, $0.22$ for QED) measures the residual slack in this correspondence.

\subsection{Direction Conditioning}
\label{sec:direction_conditioning}

A two-layer GELU MLP $\mathrm{DirEmb}(\cdot)$ maps $s$ to a $d$-dimensional embedding injected additively alongside time and positional embeddings, $\mathbf{h}_0 = \mathbf{x}_t + \mathrm{PE}(i) + \mathrm{TimeEmb}(t) + \mathrm{DirEmb}(\hat{s})$, with $s$ normalised by training-set statistics $(\mu_s, \sigma_s)$ as $\hat{s} = (s - \mu_s)/\sigma_s$.

\subsection{Training}
\label{sec:training}

Let $\mathrm{Proj}\colon \mathbb{R}^d \to \mathbb{R}^V$ denote the linear projection to the token vocabulary of size $V$ (weights tied to the token embedding table), and $\mathrm{tokens}(\mathbf{x})$ the ground-truth token sequence for molecule $\mathbf{x}$.
The network predicts the endpoint $\hat{\mathbf{x}}_1$ (velocity recovered as $(\hat{\mathbf{x}}_1 - \mathbf{x}_t)/(1-t)$) under a combined loss $\mathcal{L} = \lVert \hat{\mathbf{x}}_1 - \mathbf{x}\rVert^2 + \lambda\,\mathrm{CE}(\mathrm{Proj}(\hat{\mathbf{x}}_1), \mathrm{tokens}(\mathbf{x}))$ ($\lambda{=}1$); $\mathcal{L}_{\mathrm{FM}}$ is computed at \emph{all} sequence positions, including padding (§\ref{sec:ablation}).
The MSE term carries the flow-matching signal in continuous embedding space; the CE term anchors decoded tokens in the discrete vocabulary, without which $\hat{\mathbf{x}}_1$ drifts into regions that match the target in $L^2$ but do not decode to valid molecules.
Each epoch resamples noise for all $N$ molecules, sorts by Eq.~\eqref{eq:sorting_key}, and builds the monotone coupling of Eq.~\eqref{eq:monotone_coupling}; stratified sampling covers the full property range per mini-batch (Algorithm~\ref{alg:training}).

\definecolor{ourchange}{RGB}{132,0,0}

\begin{algorithm}[t]
\caption{\textbf{Reward Transport: Training.} Steps in \textcolor{ourchange}{\textbf{red}} are our additions; all other steps are standard flow-matching training.}
\label{alg:training}
\begin{algorithmic}[1]
\REQUIRE Dataset $\{(\mathbf{x}_n, y_n)\}_{n=1}^N$, direction calibration $(\mu_s, \sigma_s)$
\FOR{each epoch}
    \STATE Sample noise: $\mathbf{z}_n \sim \mathcal{N}(\mathbf{0}, \mathbf{I})$ for $n = 1, \ldots, N$
    \STATE \textcolor{ourchange}{Compute keys: $s_n = \lVert \bar{\mathbf{z}}_n \rVert_2$} \hfill \COMMENT{\textcolor{ourchange}{Eq.~\eqref{eq:sorting_key}}}
    \STATE \textcolor{ourchange}{Monotone coupling: pair $\operatorname{argsort}(s_n) \leftrightarrow \operatorname{argsort}(y_n)$} \hfill \COMMENT{\textcolor{ourchange}{Eq.~\eqref{eq:monotone_coupling}}}
    \FOR{each mini-batch $\{(\mathbf{z}_j, \mathbf{x}_j, y_j)\}$}
        \STATE Sample $t \sim \mathcal{U}[0,1]$;\; $\mathbf{x}_t \gets (1 - t)\,\mathbf{z}_j + t\,\mathbf{x}_j$ \hfill \COMMENT{Eq.~\eqref{eq:interpolation}}
        \STATE \textcolor{ourchange}{$\hat{s}_j \gets (s_j - \mu_s) / \sigma_s$;} \;\; $\hat{\mathbf{x}}_1, \mathrm{logits} \gets f_\theta(\mathbf{x}_t, t, \textcolor{ourchange}{\hat{s}_j})$
        \STATE $\mathcal{L} \gets \lVert \hat{\mathbf{x}}_1 - \mathbf{x}_j \rVert^2 + \lambda \cdot \mathrm{CE}(\mathrm{logits}, \mathrm{tokens}_j)$; \;\; update $\theta$
        \STATE \textcolor{ourchange}{Re-zero PAD embedding: $\mathbf{e}_{\texttt{PAD}} \gets \mathbf{0}$} \hfill \COMMENT{\textcolor{ourchange}{prevents length shortcut, §\ref{sec:ablation}}}
    \ENDFOR
\ENDFOR
\end{algorithmic}
\end{algorithm}

\subsection{Inference}
\label{sec:inference}

The user specifies a direction signal $s^*$ and solves $\mathbf{x}_{t+\Delta t} = \mathbf{x}_t + \mathbf{v}_\theta(\mathbf{x}_t, t, s^*)\Delta t$ from $\mathbf{z}\sim\mathcal{N}(\mathbf{0},\mathbf{I})$, decoding tokens via $\arg\max\mathrm{Proj}(\mathbf{x}_1)$ at $t{=}1$.
$s^*$ is a \emph{noise-space coordinate}, not a target property value: $s^*{=}3$ selects the $99.87$th percentile of the training property distribution.

%% file: 4-experiment.tex
\section{Experiments}
\label{sec:experiments}

\subsection{Setup}
\label{sec:setup}

\paragraph{Dataset.}
We use ZINC-250K~\citep{irwin2012zinc}: 224,568 drug-like molecules split 90/5/5 into train/val/test.
We evaluate on two properties with distinct chemical semantics: QED (drug-likeness, $\mu{=}0.73$) and logP (lipophilicity, $\mu{=}2.46$).
Molecules are represented as SELFIES~\citep{krenn2020selfies} (vocabulary 110, max length 72).

\paragraph{Architecture and training.}
We use a 6-layer Transformer ($d{=}768$, 8 heads, ${\sim}$50M params) with $\hat{x}_1$-prediction, trained with AdamW (lr${=}10^{-4}$).
The direction MLP adds ${\sim}$591K params (${\sim}$1\%).
Output projection is weight-tied with the token embedding.
We adopt a two-stage procedure: a base model trained for 120 epochs without OT coupling, followed by fine-tuning for 5--10 epochs with OT coupling and direction conditioning.
We report steerability ($\Delta$, $\rho$) throughout; absolute property levels are bounded by base-model quality and are discussed in §\ref{sec:discussion} (Appendix~\ref{app:qed_results}).

\paragraph{Evaluation.}
We set $s$ directly at inference (no rejection sampling), integrate with 50 Euler steps, and generate $n{=}1000$ molecules per $s$ value.
We report $\Delta$ vs.\ unconditional baseline ($s{=}0$), one-sided Welch's $t$-test, and Cohen's $d$.
All SMILES are validated with RDKit.

\subsection{Main Results}
\label{sec:results}

Section~\ref{sec:method} specified how the coupling enters training. We now ask whether it \emph{transfers} to the flow. The central question is whether the training-time coupling actually shapes the learned flow or merely imposes a correlation the model ignores.

\begin{figure*}[h!]
\centering
\includegraphics[width=\textwidth]{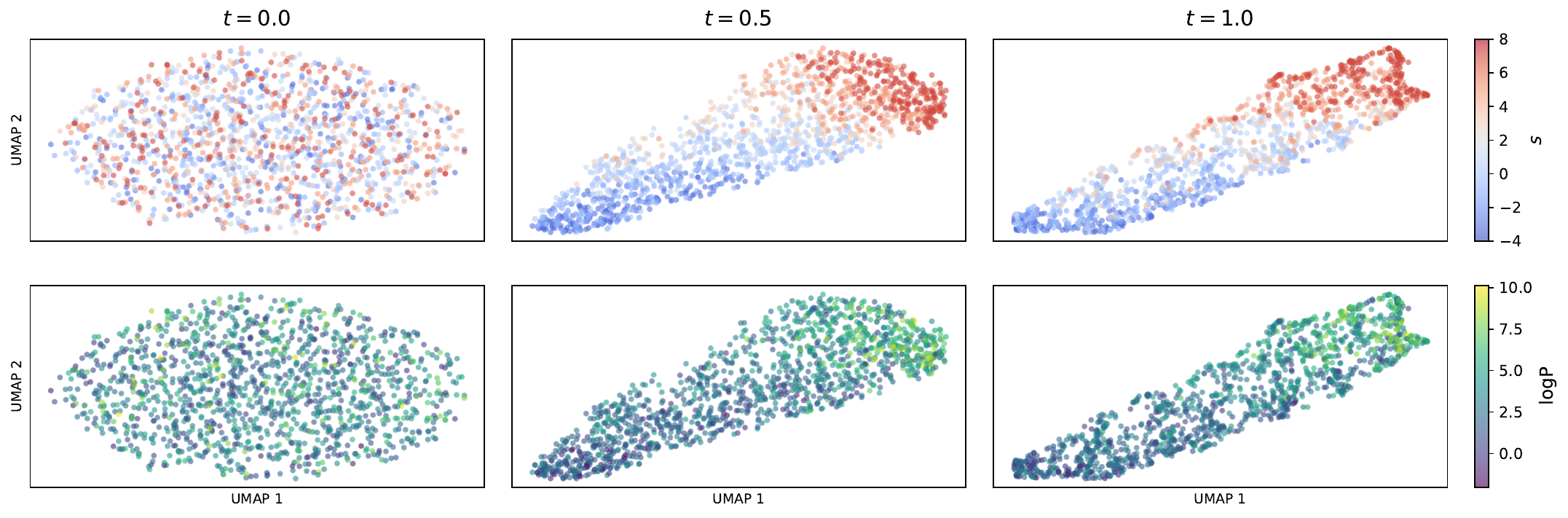}
\vspace{-0.2in}
\caption{\textbf{Rank preservation emerges through the flow.}
UMAP of $\mathbf{x}_t$ for the logP model ($n{=}1{,}400$); top row coloured by $s$, bottom by actual logP.
Structure is absent at $t{=}0$ and aligned by $t{=}1$ (per-molecule $\rho{=}0.57$, $p{<}10^{-120}$); a no-OT base model shows no alignment (Appendix~\ref{app:negative_control}).}
\label{fig:umap_trajectory}
\end{figure*}

\paragraph{Does the coupling transfer to the flow field?}
Figure~\ref{fig:umap_trajectory} provides direct evidence: at $t{=}0$ the noise carries no spatial property information; by $t{=}1$ the $s$-gradient and the logP-gradient are visually aligned (per-molecule $\rho{=}0.57$, $p{<}10^{-120}$).
Per-molecule $\rho$ directly reports the flow's coupling-preservation (Proposition~\ref{prop:cem}): $\rho_{\text{per}}{=}0.57$ (logP), $0.22$ (QED) --- strong but imperfect, and exactly the regime in which Proposition~\ref{prop:cem} becomes a distributional (rather than deterministic) statement.
The gap between per-molecule $\rho{=}0.57$ and group-mean $\rho{=}1.000$ (Table~\ref{tab:main}) is exactly the distributional statement of Proposition~\ref{prop:cem}: the coupling shifts the property \emph{distribution} without deterministically placing individual molecules.

\paragraph{How precisely does $s$ control molecular properties?}
Table~\ref{tab:main} quantifies the effect: sweeping $s\in[{-}3,{+}7]$ shifts logP from $1.50$ to $5.44$ ($+137\%$, $d{=}1.63$) and QED by $\Delta{=}{+}0.059$ ($d{=}0.36$, $p{<}10^{-9}$); both reach $\rho{=}1.000$ with $100\%$ validity and ${>}99\%$ uniqueness throughout.
Figure~\ref{fig:main_results} shows that this is a genuine distributional shift, not merely a mean change.
Crucially, the same $s$ knob grows logP molecules ($12{\to}23$ atoms) but shrinks QED molecules ($23{\to}16$)---opposite structural responses that rule out a generic size bias and reveal a property-specific program.

\begin{table*}[h!]
\centering
\scalebox{0.82}{
\setlength{\tabcolsep}{3pt}
\begin{tabular}{@{}r ccccc c ccccc@{}}
\toprule
& \multicolumn{5}{c}{\textbf{logP} (lipophilicity $\uparrow$)}
& \phantom{a}
& \multicolumn{5}{c}{\textbf{QED} (drug-likeness $\uparrow$)} \\
\cmidrule(lr){2-6} \cmidrule(lr){8-12}
$s$ & Mean & $\Delta$ (\%) & $d$ & Atoms & Val / Uniq (\%)
    & & Mean & $\Delta$ (\%) & $d$ & Atoms & Val / Uniq (\%) \\
\midrule
$-3$ & 1.50 & \cellcolor{neg37}$-$.81$^{***}$ ($-$35\%) & $-$.56 & 12.1 & 100 / 99.3
     & & .336 & \cellcolor{neg20}$-$.083$^{***}$ ($-$20\%) & $-$.47 & 23.3 & 100 / 99.1 \\
$-1$ & 2.02 & \cellcolor{neg15}$-$.28$^{***}$ ($-$12\%) & $-$.20 & 13.2 & 100 / 99.5
     & & .381 & \cellcolor{neg5}$-$.038$^{***}$ ($-$9\%) & $-$.22 & 19.8 & 100 / 99.6 \\
\rowcolor{basegray}
~~0  & 2.30 & --- & --- & 14.0 & 100 / 99.4
     & & .418 & --- & --- & 17.6 & 100 / 99.2 \\
~~1  & 2.62 & \cellcolor{pos10}$+$.32$^{***}$\; ($+$14\%) & $+$.21 & 15.2 & 100 / 99.6
     & & .442 & \cellcolor{pos5}$+$.023$^{**}$\; ($+$6\%) & $+$.14 & 16.5 & 100 / 99.3 \\
~~3  & 3.43 & \cellcolor{pos30}$+$1.12$^{***}$ ($+$49\%) & $+$.70 & 17.6 & 100 / 99.6
     & & .452 & \cellcolor{pos5}$+$.033$^{***}$ ($+$8\%) & $+$.20 & 15.8 & 100 / 99.4 \\
~~5  & 4.43 & \cellcolor{pos100}$+$2.13$^{***}$ ($+$93\%) & $+$1.23 & 20.4 & 100 / 99.6
     & & .477 & \cellcolor{pos10}$+$.059$^{***}$ ($+$14\%) & $+$.36 & 16.1 & 100 / 99.5 \\
~~7  & 5.44 & \cellcolor{pos100}$+$3.14$^{***}$ ($+$137\%) & $+$1.63 & 23.4 & 100 / 99.6
     & & .477 & \cellcolor{pos10}$+$.058$^{***}$ ($+$14\%) & $+$.35 & 16.4 & 100 / 99.5 \\
\midrule
\multicolumn{1}{@{}r}{$\rho_{\text{per}}(s,y)$}        & \multicolumn{2}{c}{$\mathbf{0.57}~(p{<}10^{-120})$} & & &
& & \multicolumn{2}{c}{$\mathbf{0.22}~(p{<}10^{-10})$}  \\
\multicolumn{1}{@{}r}{$\rho_{\text{group}}(s,\bar{y})$} & \multicolumn{2}{c}{$1.000~(p < 10^{-6})$} & & &
& & \multicolumn{2}{c}{$1.000~(p < 10^{-6})$} \\
\bottomrule
\end{tabular}
}
\caption{\textbf{Zero-shot property control via noise-space direction signal $s$.}
A single scalar $s$ monotonically steers both logP and QED with opposite structural responses: higher $s$ produces \emph{larger} molecules for logP ($12{\to}23$ atoms) but \emph{smaller} molecules for QED ($23{\to}16$ atoms), ruling out a generic size bias.
The per-molecule rank correlation $\rho_{\text{per}}(s,y)$ directly measures how well the learned flow preserves the training-time coupling (Proposition~\ref{prop:cem}); the group-mean statistic $\rho_{\text{group}}(s,\bar{y})$ is the coarser $s$-bin$\to$mean-property monotonicity. The gap between the two is the distributional CEM-truncation regime the proposition predicts.
Stronger base model reaches QED$=$0.518 (Appendix~\ref{app:qed_results}).
$n{=}1000$ per condition; $d$: Cohen's effect size.
$^{***}\!p{<}.001$, $^{**}\!p{<}.01$.}
\label{tab:main}
\end{table*}

\begin{figure}[t]
\centering
    \includegraphics[width=\textwidth]{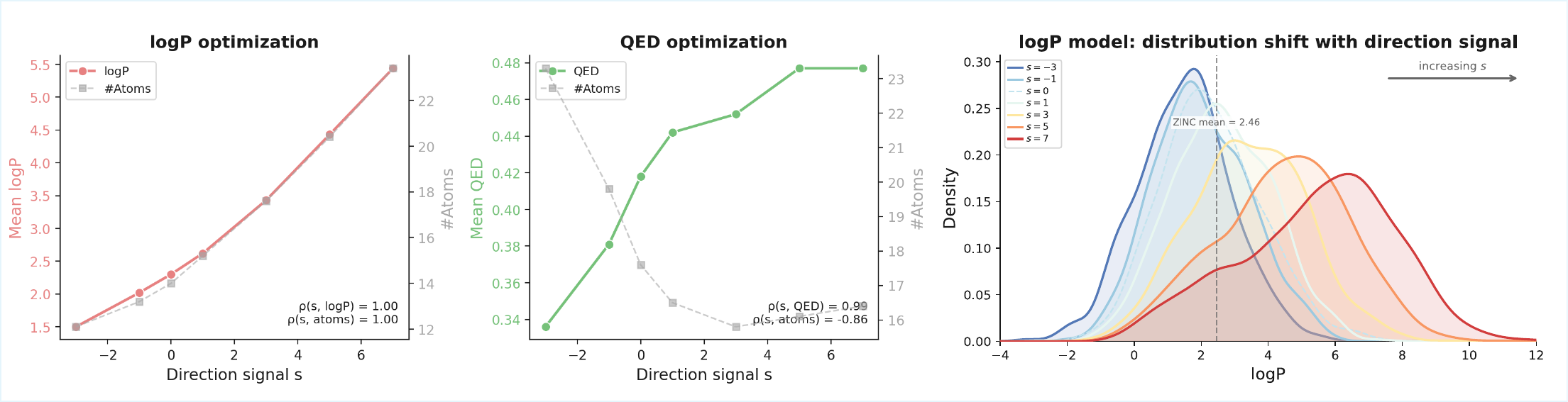}
\caption{\textbf{Property control and distribution shift.}
\textbf{Left:} Mean property vs.\ $s$ for logP and QED ($n{=}1000$ per $s$).
Gray dashed: atom count.
Both achieve $\rho{=}1.000$ with opposite size responses.
\textbf{Right:} Full logP distributions across seven $s$ values, shifting continuously rightward.
Dashed line: ZINC mean ($\text{logP}{=}2.46$).
QED distributions in Appendix~\ref{app:qed_results}.}
\label{fig:main_results}
\end{figure}

\paragraph{Replication on a second dataset.}
\label{par:guacamol}
We repeat the experiment on GuacaMol~\citep{brown2019guacamol} ($1.6$M molecules) with identical hyperparameters (Table~\ref{tab:dataset_summary}).
logP monotonicity reproduces with $\rho{=}1.000$ and a \emph{larger} effect size ($d{=}2.31$ at $s{=}5$ vs.\ $1.23$ on ZINC), consistent with the percentile-truncation mechanism of Proposition~\ref{prop:cem}: a broader empirical property range gives the sorted coupling more room to stretch.
QED transfers as a significant but weaker trend ($\rho{=}0.789$, monotone on $s\in[{-}3,{+}4]$, $p{=}8{\times}10^{-4}$), degrading at extreme $s$ consistent with QED's bounded geometry.
The opposite-size disambiguation is preserved in the monotone regime: logP grows molecules $10.5{\to}39.4$ atoms, QED shrinks them $22.6{\to}13.4$, ruling out a ZINC-specific explanation for the property-specific program.
Full sweep tables and a seed-stability replication are in Appendix~\ref{app:guacamol}.

\begin{table}[h]
\centering
\small
\caption{\textbf{Replication of property control across datasets.}
Group-mean rank correlation $\rho$, property shift $\Delta$ at matched $s{=}3$, and validity. Full sweeps in Appendix~\ref{app:guacamol}.}
\label{tab:dataset_summary}
\setlength{\tabcolsep}{6pt}
\begin{tabular}{@{}l cccc@{}}
\toprule
 & \multicolumn{2}{c}{\textbf{ZINC-250K}} & \multicolumn{2}{c}{\textbf{GuacaMol ($1.6$M)}} \\
\cmidrule(lr){2-3} \cmidrule(lr){4-5}
Metric & logP & QED & logP & QED \\
\midrule
$\rho(s,\bar{y})$                          & $1.000$  & $1.000$  & $\mathbf{1.000}$ & $0.789^{*}$ \\
$\Delta$ at $s{=}3$                        & $+1.12$  & $+0.033$ & $+2.81$          & $+0.057$ \\
Cohen's $d$ at $s{=}3$                     & $0.70$   & $0.20$   & $1.65$           & $0.36$ \\
Validity / Uniqueness (\%)                 & $100/99.6$ & $100/99.4$ & $100/99.8$ & $100/99.8$ \\
Atoms, $s{=}{-}3 \to {+}3$                 & $12 \to 18$ & $23 \to 16$ & $10 \to 24$ & $23 \to 15$ \\
\bottomrule
\end{tabular}
\vspace{2pt}
\begin{flushleft}
\footnotesize $^{*}$GuacaMol QED is monotone for $s \in [-3, +4]$; control degrades at $s{=}7$ (see Appendix~\ref{app:guacamol}).
\end{flushleft}
\end{table}

\paragraph{Same-backbone comparison against Conditional Flow Matching.} Replication confirms the effect is not ZINC-specific. We now ask whether it is not \emph{mechanism}-specific either---whether the same backbone supports property control through conditioning alone.
On the \emph{identical} backbone and training budget, a Conditional FM baseline that injects the property as a conditioning scalar also attains $\rho{=}1.000$ with matched validity and diversity ($0.905$ vs.\ $0.909$), and reaches a wider logP span at the extremes ($1.79{\to}7.99$ vs.\ $1.50{\to}5.50$)---expected, since it is trained to hit explicit targets, whereas Reward Transport steers a distribution via percentile truncation (Prop.~\ref{prop:cem}).
The baseline is not intended to show that Reward Transport dominates conditioning; it shows that coupling alone can support a distinct distribution-level control semantics on a backbone where ordinary conditioning also works---different points on the steerability--interpretability Pareto frontier (Appendix~\ref{app:conditional_comparison}).

Both observations point to one mechanism: the training-time coupling writes a property rank into the flow field that is then transported to the generated samples.
We verify this directly.

\paragraph{Rank preservation as a training dynamic.}
\begin{wrapfigure}{r}{0.44\linewidth}
    \vspace{-1.2em}
    \centering
    \includegraphics[width=\linewidth]{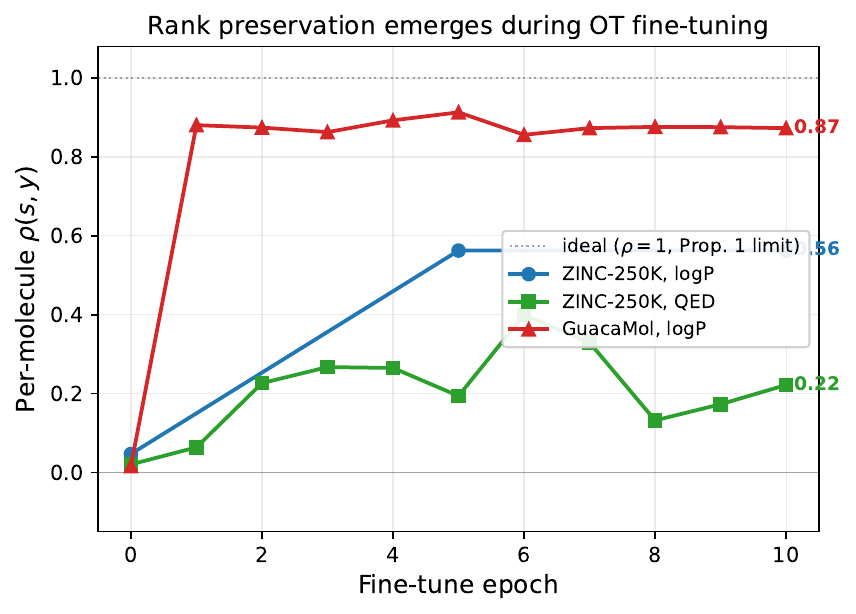}
    \vspace{-1.8em}
    \caption{\textbf{Rank preservation emerges within the first OT fine-tuning epoch.}
    Per-molecule $\rho(s, y)$ across training checkpoints; $\rho{\approx}0$ at base, lifted within one epoch, plateau ${<}1$ is the distributional gap Prop.~\ref{prop:cem} describes.}
    \label{fig:rank_dynamics}
    \vspace{-1.0em}
\end{wrapfigure}
Proposition~\ref{prop:cem} describes the coupling-preserving limit $\rho_{\text{per}}\to 1$; Figure~\ref{fig:rank_dynamics} shows that the realised $\rho_{\text{per}}$ is not an ex-post interpretation but an emergent quantity that appears within the first OT fine-tuning epoch and plateaus at a dataset/property-specific value strictly below $1$.
On GuacaMol logP, $\rho_{\text{per}}$ lifts from $0.02$ (base) to $0.87$ in a single epoch, then oscillates; ZINC-250K QED plateaus near $0.22$ by epoch $2$.
Group-mean $\rho \geq 0.96$ across all three curves from epoch $1$, so the large per-molecule residual is exactly the distributional approximation gap Proposition~\ref{prop:cem} formalises.

\paragraph{Further checks.}
Ten-descriptor profiling (Appendix~\ref{app:cross_property}) shows the two targets induce \emph{opposite}, chemically coherent programs (QED: $\rho_{\text{TPSA}}{=}{-}1.00$; logP: $+.89$) consistent with each property's construction~\citep{bickerton2012quantifying}.
An extended molecular-quality audit on the same $s$-sweeps (Appendix~\ref{app:extended_eval}) shows validity, uniqueness, and novelty flat at $\geq 99\%$; scaffold diversity and SA follow property-specific trajectories (the logP sweep degrades SA $3.97\to 5.32$, a known scalar-logP tradeoff also observed under Conditional FM), and FCD rises mildly at extrapolation extremes in the same $29$--$40$ range as the same-backbone Conditional FM baseline.

\subsection{Failure Modes and Ablation}
\label{sec:ablation}

Having established that the coupling transfers and reproduces across datasets and architectural conditions (§\ref{sec:results}), we next ask \emph{which components of the recipe are load-bearing}, and what breaks when each is removed. Applying continuous OT coupling to discrete sequences reveals two previously unreported failure modes that would affect any flow model in this regime.
Table~\ref{tab:ablation} summarizes the component analysis; a separate robustness check (Table~\ref{tab:key_ablation_main}) shows the method is insensitive to the choice of 1D key.
We then discuss the two discoveries in turn.

\begin{table}[h]
\centering
\small
\begin{tabular}{@{}lccc ccc@{}}
\toprule
& OT & Dir & Unmask & $\Delta$QED ($s{=}5$) & $\rho(s,\text{QED})$ & $\rho(s,\text{atoms})$ \\
\midrule
Full method           & $\checkmark$ & $\checkmark$ & $\checkmark$ & $+$0.059 & 1.000 & $-$0.85 \\
$-$ OT coupling       & $\times$     & $\checkmark$ & $\checkmark$ & $-$0.023 & $-$0.42 & --- \\
$-$ Direction MLP     & $\checkmark$ & $\times$     & $\checkmark$ & $-$0.023 & ---     & --- \\
$-$ Unmasked MSE      & $\checkmark$ & $\checkmark$ & $\times$     & $-$0.006 & 0.65    & $+$0.95 \\
\bottomrule
\end{tabular}
\vspace{0.1in}
\caption{\textbf{Component ablation.}
Removing any single component eliminates property control.
OT coupling is the active ingredient ($\times$\,OT: $\rho{=}{-}0.42$); the direction MLP provides the signal pathway ($\times$\,Dir: $\Delta$ reverses sign); unmasked MSE prevents the length shortcut ($\times$\,Unmask: $s$ controls atoms, not QED).
$n{=}1000$; $\Delta$QED at $s{=}5$.}
\label{tab:ablation}
\end{table}
\vspace{-0.15in}

\paragraph{The 1D key is not privileged.}
A natural concern is whether $\lVert\bar{\mathbf{z}}\rVert_2$ is a tuned or magical choice.
Retraining with three alternatives---an $\ell_\infty$ norm, a random linear projection, and a single coordinate---while holding the backbone and training budget fixed (Table~\ref{tab:key_ablation_main}) shows all four keys achieve group-level $\rho{=}1.000$ at $100\%$ validity, with per-molecule $\rho$ in a narrow band $[0.58, 0.71]$.
$\ell_\infty$ and the random projection attain slightly higher per-molecule $\rho$ than $\ell_2$, within noise at $n{=}500$/$s$; the monotone rearrangement, not the specific scalar summary, carries the coupling signal (Appendix~\ref{app:key_ablation}).

\begin{table}[h]
\centering
\small
\setlength{\tabcolsep}{8pt}
\begin{tabular}{@{}l ccc@{}}
\toprule
Sorting key $s(\mathbf{z})$ & Group $\rho$ & Per-molecule $\rho$ & Cohen's $d_{s=3\,\text{vs}\,s=0}$ \\
\midrule
$\lVert\bar{\mathbf{z}}\rVert_2$ (main)  & $1.000$ & $0.58$ & $0.66$ \\
$\lVert\bar{\mathbf{z}}\rVert_\infty$    & $1.000$ & $0.71$ & $1.06$ \\
Random linear projection                 & $1.000$ & $0.71$ & $1.17$ \\
Single coordinate                        & $1.000$ & $0.63$ & $0.96$ \\
\bottomrule
\end{tabular}
\vspace{2pt}
\caption{\textbf{The 1D sorting key is not privileged.}
All four scalar summaries of $\bar{\mathbf{z}}$ achieve group-level monotone steering at $100\%$ validity; per-molecule $\rho$ sits in a narrow band.
ZINC-250K logP, $n{=}500$/$s$, same backbone and training budget as the main experiment.
Full sweep in Appendix~\ref{app:key_ablation}.}
\label{tab:key_ablation_main}
\end{table}

\paragraph{Discovery 1: Implicit signal erasure.}
Without the direction MLP ($\times$\,Dir), the model cannot access $s$, and the property shift reverses sign ($\Delta\text{QED}{=}{-}0.023$).
The root cause is architectural: Pre-LayerNorm Transformers normalize activations before each layer, erasing the magnitude information that the OT coupling encodes into $\|\bar{\mathbf{z}}\|$.
The direction MLP resolves this by mapping $s$ to an activation \emph{pattern}---a direction in embedding space---that survives normalization.
This finding generalizes within scope: \emph{in Pre-LayerNorm (or similar) architectures with norm-based couplings on discrete sequences, any approach that encodes information into the magnitude of flow matching noise will face the same erasure; Post-LayerNorm or learned-scalar injection schemes may avoid it.}
Replacing OT coupling with random pairing ($\times$\,OT) confirms that even with the direction MLP intact, the training-time coupling is the active ingredient ($\rho{=}{-}0.42$). A second failure mode appears once the coupling \emph{is} reaching the model: the model can still re-route the signal into an unintended feature.

\paragraph{Discovery 2: The length confound.}
Masking the MSE loss at padding positions ($\times$\,Unmask) produces a striking failure: $s$ controls molecular \emph{length} ($\rho(s, \text{atoms}){=}{+}0.95$) rather than QED ($\Delta\text{QED}{=}{-}0.006$, n.s.).
Under masked MSE, the model takes a geometric shortcut---extending sequences is the lowest-resistance response to high-$s$ noise, because longer sequences trivially have higher noise norms.
Computing $\mathcal{L}_{\mathrm{FM}}$ at padding positions breaks this shortcut: since PAD embeddings are zero, unmasked MSE forces the model to regress the \emph{raw noise vector} at every padding position, anchoring the global noise magnitude in its hidden representations---exactly the information that Pre-LayerNorm would otherwise erase.
Extending sequences now costs the readout signal, closing the length shortcut (Appendix~\ref{app:ablation_figs}).

\paragraph{Training dynamics.}
From-scratch training exhibits a transient \emph{signal window}: OT steering emerges at epoch~3, peaks at epoch~7, and is overwritten by epoch~10.
Warm-start fine-tuning keeps the signal significant throughout ($p{<}10^{-3}$ at every epoch) with a $60\%$ higher peak effect size; separating ``learn to generate'' from ``learn to steer'' prevents the generation objective from erasing the coupling.
Stronger bases yield higher absolute QED but smaller marginal $\Delta$, as Proposition~\ref{prop:cem} predicts (Appendix~\ref{app:qed_results}).

%% file: 6-discussion.tex

\section{Discussion}
\label{sec:discussion}

\paragraph{Scope and boundaries.}
Reward Transport gives a one-dimensional distributional knob: it truncates the generated distribution by property rather than assigning exact per-sample values (§\ref{sec:results}).
This shapes both what the interface encodes and where it fails.
On topological properties such as synthetic accessibility, a scalar $\lVert\bar{\mathbf{z}}\rVert_2$ key cannot distinguish topologically simple from complex molecules of similar size, and we observe no meaningful SA control ($|\rho|{<}0.4$; Appendix~\ref{app:sa}).
Extending to multi-property or topology-aware control would require a higher-dimensional coupling key whose inference-time semantics no longer reduce to 1D monotone rearrangement~\citep{villani2003topics}.

The interface also depends on the prediction target.
Under $\hat{\mathbf{x}}_1$-prediction (and the affine-equivalent velocity prediction), the Bayes-optimal target of the network depends on the coupled molecule, so the coupling-induced gradient routes into the Direction MLP.
Under $\varepsilon$-prediction~\citep{ho2020denoising}, the target is a fresh $\boldsymbol{\varepsilon} \sim \mathcal{N}(\mathbf{0}, \mathbf{I})$ independent of the coupling, and the coupling-induced gradient is attenuated by the forward-process signal-to-noise ratio $\alpha_t/\sigma_t$; we observe no measurable coupling effect on QM9+EDM~\citep{hoogeboom2022equivariant}, and Appendix~\ref{app:eps_grad} derives the attenuation analytically.
Coupling-level alignment is therefore available only when the training target depends strongly on the coupled pair.

Reward Transport is complementary to conditional generation: conditioning injects a per-sample target value at training and inference, while coupling imposes a distribution-level rank correspondence at training only.
Our same-backbone Conditional FM experiment (§\ref{sec:results}, Appendix~\ref{app:conditional_comparison}) shows both control modes work in isolation; composing them inside a single trained model, exposing two orthogonal inference-time knobs, is a natural next step.

\paragraph{Limitations.}
Our empirical claims scope to two SELFIES-tokenized datasets (ZINC-250K~\citep{irwin2012zinc} and GuacaMol~\citep{brown2019guacamol}), one property at a time, and a single backbone; unconditional QED ($0.477$ on ZINC) sits below the dataset mean, so absolute property values are base-model-bounded even though steerability is not (Appendix~\ref{app:qed_results}).
On GuacaMol, logP steering is monotone across the full sweep while QED control is restricted to $s \in [{-}3, {+}4]$ and degrades at extreme $s$ (Appendix~\ref{app:guacamol}), consistent with QED's bounded geometry.
We have not tested larger backbones, multi-property targets, or reinforcement-learning fine-tuning on top of a coupling-aligned flow.
Together these boundaries sharpen rather than weaken the reframe: the coupling is an alignment interface precisely because it has scope---working under specific target parameterisations, carrying scalar-distributional properties, and failing predictably where these conditions break.

%% file: 7-conclusion.tex

\section{Conclusion}
\label{sec:conclusion}
The coupling in flow matching has long been chosen to make training easier; we argued it is instead a distribution-level alignment interface, complementary to classifier-free guidance, reward models, and preference optimization.
Reward Transport instantiates this as a monotone rearrangement of noise against a scalar property, whose coupling-preserving limit induces the same truncated distribution as one Cross-Entropy Method selection step (Proposition~\ref{prop:cem}); experiments on ZINC-250K and GuacaMol---replicated across seeds---produce monotone group-level control for logP and QED, and opposite structural responses under the same knob that rule out generic size bias.
The 3D negative result and its signal-to-noise derivation further delineate where the interface exists: the training target's Bayes-optimal predictor must depend on the coupled pair. The interface view opens a design space rather than closing one: the \emph{coupling cost} (scalar monotone here; multi-target or learned rankings elsewhere), the \emph{prediction target} (our analysis identifies $\hat{\mathbf{x}}_1$/velocity as admitting the interface and $\varepsilon$ as not), and \emph{composition} with conditioning (our same-backbone comparison suggests the two are combinable rather than substitutes).
Reward Transport occupies one point in this space, chosen for clarity rather than generality; the framing makes clear what is worth changing. What Reward Transport establishes is not that coupling outperforms conditioning, but that coupling is a distinct, distribution-level control channel that can carry alignment structure on its own.

%% file: 5-appendix.tex
\appendix

\section*{Broader Impact}
\label{app:broader_impact}
 
Reward Transport reduces the computational cost of property-controlled molecular generation by eliminating the need for oracle calls, reward models, or iterative optimization at inference time. This could accelerate early-stage drug discovery and materials screening by enabling rapid exploration of chemical space along desired property axes.
 
The method's controllability is bounded by the training data distribution: the $s$ knob steers generation within (and modestly beyond) the property range present in training molecules, and does not confer the ability to generate arbitrary molecular structures on demand. Nonetheless, any method that facilitates targeted molecular generation carries dual-use risk---the same capability that enables drug optimization could, in principle, be directed toward synthesizing harmful compounds. We note that this risk is not unique to our approach; it applies equally to conditional generation, guided diffusion, and RL-based molecular optimization methods that are already publicly available. The primary mitigation is that translating computationally generated molecules into real-world harm requires wet-lab synthesis expertise and resources that are independent of the generative model.
 
We do not foresee direct fairness or privacy concerns arising from this work, as it operates on molecular representations rather than data about individuals.

\section{Proof of Proposition 1}
\label{app:proof}

\begin{proposition}[Property-Aligned OT as Implicit CEM, restated]
Let $\pi^*$ be the monotone coupling between $s(\mathbf{z})$ and $y(\mathbf{x})$, and let $F_S$, $F_Y$ denote their respective CDFs.
Let $\rho_{\text{per}} = \mathrm{corr}_{\text{Spearman}}\!\bigl(s(\mathbf{z}), y(\mathbf{x}_\theta(\mathbf{z}))\bigr)$ denote the per-molecule rank correlation between the noise key and the generated property---a direct measure of how well the learned flow preserves the coupling.
In the coupling-preserving limit $\rho_{\text{per}}{\to}1$, for any threshold $\tau$:
\begin{equation}
    p_\theta\!\left(\mathbf{x} \mid s(\mathbf{z}) \geq \tau\right)
    \;\approx\;
    \frac{p_1(\mathbf{x})\;\mathbf{1}\!\bigl[y(\mathbf{x}) \geq y_\tau\bigr]}{Z(\tau)},
    \qquad
    y_\tau = F_Y^{-1}\!\bigl(F_S(\tau)\bigr),
\end{equation}
where $Z(\tau) = 1 - F_Y(y_\tau)$.
The realised $\rho_{\text{per}}$ (reported in §\ref{sec:results}: $0.57$ for logP, $0.22$ for QED) quantifies where the learned flow lands relative to the CEM-truncation semantics; smaller $\rho_{\text{per}}$ implies a distributional approximation rather than a pointwise equality.
\end{proposition}

\begin{proof}
The monotone coupling $\pi^*$ pairs the $\alpha$-quantile of $s$ with the $\alpha$-quantile of $y$ for all $\alpha \in [0, 1]$.
Formally, if $s_\alpha = F_S^{-1}(\alpha)$ and $y_\alpha = F_Y^{-1}(\alpha)$, then under $\pi^*$, the noise vector with sorting key $s_\alpha$ is paired with the molecule having property $y_\alpha$.

Since $s$ and $y$ are both one-dimensional, this rank-preserving pairing is the unique optimal transport plan under any convex cost $c(a, b) = h(|a - b|)$ with $h$ convex and increasing (monotone rearrangement theorem; \citealp{villani2003topics}).

Now consider conditioning on $s(\mathbf{z}) \geq \tau$.
Let $\alpha_\tau = F_S(\tau)$, so that $\Pr[s \geq \tau] = 1 - \alpha_\tau \triangleq \varepsilon$.
Under the monotone coupling, the set $\{s \geq \tau\}$ maps exactly to the set $\{y \geq y_\tau\}$ where $y_\tau = F_Y^{-1}(\alpha_\tau)$.

If the flow model has converged such that the learned mapping $\mathbf{z} \mapsto \mathbf{x}$ preserves this rank correspondence (i.e., a noise vector with $s(\mathbf{z}) = s_\alpha$ is transported to a molecule with $y(\mathbf{x}) \approx y_\alpha$), then:
\begin{align}
    p_\theta(\mathbf{x} \mid s(\mathbf{z}) \geq \tau)
    &= \frac{p_\theta(\mathbf{x}) \cdot \Pr[s(\mathbf{z}) \geq \tau \mid \mathbf{x}]}{\Pr[s(\mathbf{z}) \geq \tau]} \\
    &\approx \frac{p_1(\mathbf{x}) \cdot \mathbf{1}[y(\mathbf{x}) \geq y_\tau]}{\varepsilon} \\
    &= \frac{p_1(\mathbf{x}) \cdot \mathbf{1}[y(\mathbf{x}) \geq y_\tau]}{Z(\tau)},
\end{align}
where the approximation uses the convergence assumption: a molecule $\mathbf{x}$ with $y(\mathbf{x}) \geq y_\tau$ was generated from noise with $s(\mathbf{z}) \geq \tau$ (and vice versa), so $\Pr[s \geq \tau \mid \mathbf{x}] \approx \mathbf{1}[y(\mathbf{x}) \geq y_\tau]$.

The resulting distribution is the data distribution truncated to the top-$\varepsilon$ fraction by property value---precisely the output of one selection step of the Cross-Entropy Method~\citep{rubinstein1999cross} with elite fraction $\varepsilon = 1 - F_S(\tau)$.

\textbf{Assumptions and limitations.}
The approximation relies on two conditions: (1)~the flow model has sufficient capacity to learn the coupling-preserving mapping, and (2)~the monotone coupling between the 1D summaries $s$ and $y$ reflects the structure of the full high-dimensional mapping.
Condition (2) is inherently approximate---the 1D projection $s(\mathbf{z}) = \lVert \bar{\mathbf{z}} \rVert_2$ captures only the norm of the mean-pooled noise, not its full geometry.
In practice, we observe strong but imperfect rank preservation ($\rho = 1.000$ on aggregated $s$-sweep, Section~\ref{sec:results}).
\end{proof}

\newpage
\section{Extended QED Results}
\label{app:qed_results}

The main text reports logP as the primary result due to its stronger effect size ($d{=}1.64$) and wider dynamic range.
Here we provide complete QED results, including distribution shifts, two training configurations, and UMAP visualization.

\paragraph{Distribution shift.}
Figure~\ref{fig:qed_histogram} shows the full QED distribution at each $s$ value.
The rightward shift with increasing $s$ is visible but more modest than for logP (Figure~\ref{fig:main_results}, right), consistent with QED's narrower range on ZINC-250K.
The ZINC training mean ($\text{QED}{=}0.728$) remains above all generated distributions, reflecting the base model quality gap discussed in Section~\ref{sec:setup}.

\begin{figure}[h]
    \centering
    \includegraphics[width=0.65\linewidth]{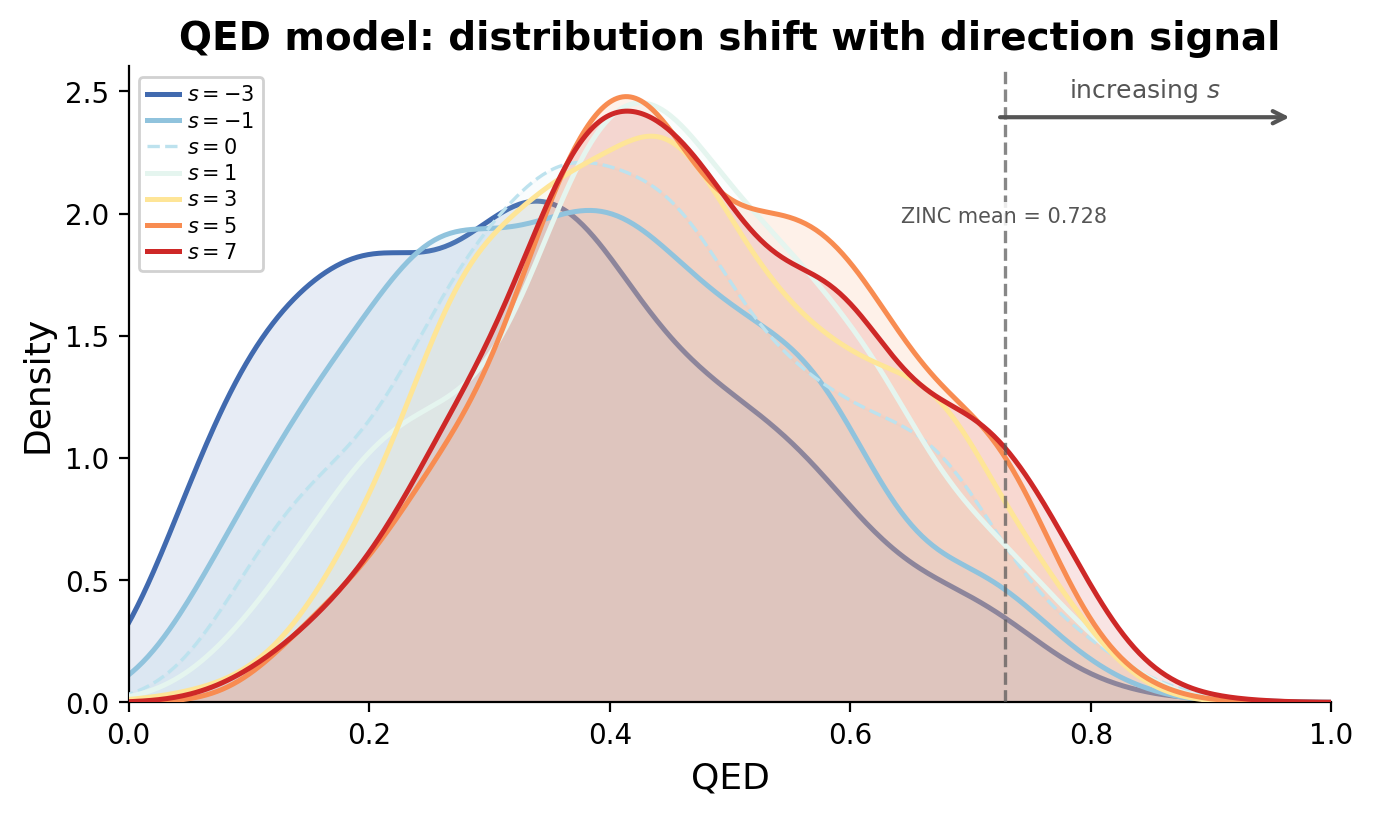}
    \caption{\textbf{QED distribution shift.}
    Generated QED distributions across all seven $s$ values ($n{=}1000$ per $s$).
    Distributions shift rightward with increasing $s$, though the effect is smaller than for logP.
    Dashed line: ZINC training mean ($\text{QED}{=}0.728$).}
    \label{fig:qed_histogram}
\end{figure}

\paragraph{Standard configuration.}
Table~\ref{tab:qed_standard} reports the full $s$-sweep for the standard configuration (OT + direction MLP + unmasked MSE, warm-start from early checkpoint, 5 epochs fine-tuning, base QED $= 0.414$).
This is the configuration reported in the main text Table~\ref{tab:main}.

\begin{table}[h]
\centering
\caption{\textbf{QED $s$-sweep, standard configuration} ($n{=}1000$).}
\label{tab:qed_standard}
\small
\begin{tabular}{@{}r cccccc@{}}
\toprule
$s$ & Mean QED & $\Delta$ & $p$-value & $d$ & Atoms & Unique \\
\midrule
$-3$ & 0.333 & $-$0.081 & $4.6{\times}10^{-24}$ & $-$0.53 & 23.2 & 991 \\
$-1$ & 0.380 & $-$0.034 & $1.7{\times}10^{-5}$  & $-$0.36 & 19.4 & 996 \\
~~0  & 0.415 & ---      & ---                     & ---     & 17.6 & 992 \\
~~1  & 0.436 & $+$0.022 & $4.0{\times}10^{-3}$   & $+$0.05 & 16.6 & 993 \\
~~3  & 0.458 & $+$0.044 & $3.3{\times}10^{-9}$   & $+$0.09 & 16.2 & 994 \\
~~5  & 0.459 & $+$0.045 & $2.7{\times}10^{-9}$   & $+$0.27 & 16.3 & 995 \\
~~7  & 0.476 & $+$0.062 & $3.3{\times}10^{-16}$  & $+$0.27 & 16.4 & 995 \\
\midrule
\multicolumn{2}{@{}l}{$\rho(s, \text{QED})$} & \multicolumn{2}{c}{$1.000~(p < 10^{-6})$} \\
\bottomrule
\end{tabular}
\end{table}

\paragraph{Strong base configuration.}
Table~\ref{tab:qed_strong} reports the same method warm-started from the 120-epoch base model (direction MLP with $10{\times}$ learning rate, 10 epochs fine-tuning, base QED $= 0.477$).

\begin{table}[h]
\centering
\caption{\textbf{QED $s$-sweep, strong base configuration} ($n{=}500$).}
\label{tab:qed_strong}
\small
\begin{tabular}{@{}r cccccc@{}}
\toprule
$s$ & Mean QED & $\Delta$ & $p$-value & $d$ & Atoms & Unique \\
\midrule
$-3$ & 0.384 & $-$0.093 & $3.1{\times}10^{-16}$ & $-$0.53 & 21.2 & 500 \\
$-2$ & 0.413 & $-$0.064 & $1.7{\times}10^{-8}$  & $-$0.36 & 19.3 & 500 \\
$-1$ & 0.460 & $-$0.016 & $0.14$                 & $-$0.09 & 17.1 & 500 \\
~~0  & 0.473 & $-$0.004 & $0.70$                 & $-$0.02 & 15.5 & 499 \\
~~1  & 0.484 & $+$0.007 & $0.46$                 & $+$0.05 & 13.7 & 500 \\
~~3  & 0.491 & $+$0.014 & $0.14$                 & $+$0.09 & 12.5 & 500 \\
~~5  & 0.506 & $+$0.029 & $3{\times}10^{-3}$     & $+$0.19 & 12.3 & 499 \\
~~7  & 0.518 & $+$0.042 & $2{\times}10^{-5}$     & $+$0.27 & 12.6 & 500 \\
\midrule
\multicolumn{2}{@{}l}{$\rho(s, \text{QED})$} & \multicolumn{2}{c}{$0.965~(p = 2.5{\times}10^{-8})$} \\
\bottomrule
\end{tabular}
\end{table}

The strong base model achieves the highest absolute QED ($0.518$ at $s{=}7$) but smaller marginal improvement ($\Delta{=}{+}0.042$ vs.\ ${+}0.062$).
This trade-off is predicted by the CEM equivalence (Proposition~\ref{prop:cem}): as the base distribution approaches the target, truncation provides diminishing marginal gain.
The rank correlation remains near-perfect ($\rho{=}0.965$), confirming that steerability is robust across base model quality.

\paragraph{UMAP rank preservation.}
Figure~\ref{fig:umap_qed} shows the same rank preservation analysis as Figure~\ref{fig:umap_trajectory} (main text) for the QED model.
The gradient alignment between $s$-coloring and QED-coloring is present but weaker than for logP, consistent with QED's smaller per-molecule rank correlation ($\rho{=}0.22$ vs.\ $0.57$).

\begin{figure}[h]
    \centering
    \includegraphics[width=0.85\linewidth]{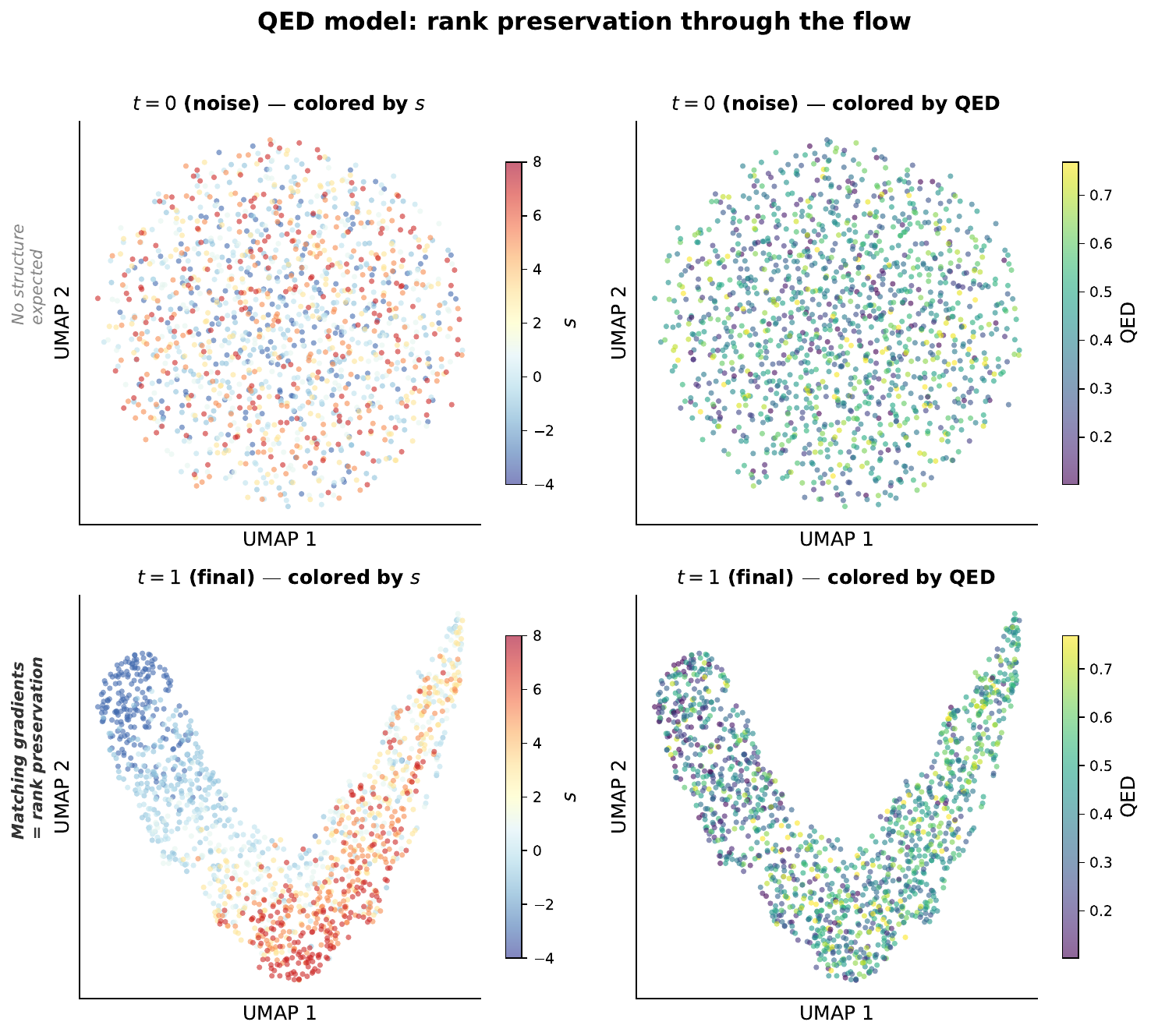}
    \caption{\textbf{Rank preservation: QED model.}
    UMAP projections of $\mathbf{x}_t$ at $t{=}0$ (top) and $t{=}1$ (bottom), colored by $s$ (left) and by actual QED (right).
    At $t{=}0$, no structure is present.
    At $t{=}1$, both gradients align (per-molecule $\rho(s, \text{QED}){=}0.22$), though with more overlap than the logP model (Figure~\ref{fig:umap_trajectory}), reflecting QED's smaller effect size.}
    \label{fig:umap_qed}
\end{figure}
\section{Unconditional Generation Quality}
\label{app:moses}

We evaluate unconditional generation quality of our base model (120 epochs, no OT coupling) using standard molecular generation metrics on $n{=}10{,}000$ generated molecules.

\begin{table}[h]
\centering
\caption{Unconditional generation quality (MOSES benchmark metrics).}
\label{tab:moses}
\small
\begin{tabular}{@{}lcc@{}}
\toprule
Metric & Ours & CharRNN$^\dagger$ \\
\midrule
Validity (\%)     & 100.0 & 97.5 \\
Unique@1k (\%)    & 99.6  & 99.9 \\
Unique@10k (\%)   & 98.7  & 99.8 \\
Novelty (\%)      & 100.0 & 84.2 \\
IntDiv1           & 0.927 & 0.856 \\
IntDiv2           & 0.917 & --- \\
\bottomrule
\end{tabular}
\begin{flushleft}
\small $^\dagger$Typical MOSES benchmark values from literature.
\end{flushleft}
\end{table}

The base model achieves 100\% validity (guaranteed by SELFIES), 100\% novelty (zero memorization), and high internal diversity (IntDiv1$=$0.927, above typical baselines).
Uniqueness at scale is slightly below autoregressive baselines (98.7\% at $n{=}10{,}000$), indicating minor mode repetition.

\section{Full Cross-Property Descriptor Table}
\label{app:cross_property}

Table~\ref{tab:cross_property} reports rank correlations $\rho(s, \text{descriptor})$ between the direction signal $s$ and ten molecular descriptors for both the QED-optimised and logP-optimised models ($n{=}1000$ per $s$, Spearman, bold indicates $p{<}.01$).
The two property targets induce \emph{opposite}, chemically coherent programs---mirroring the known construction of each score---while leaving synthesis-relevant features (SA score, frac\_sp3) essentially untouched.

\begin{table}[h]
\centering
\small
\caption{\textbf{Cross-property descriptor response.} $\rho(s, \text{property})$ for the QED-optimised and logP-optimised models. Bold: $p{<}.01$.}
\label{tab:cross_property}
\begin{tabular}{@{}l rr@{}}
\toprule
Descriptor  & QED model & logP model \\
\midrule
QED         & $\mathbf{+.93}$  & $+.18$ \\
logP        & $-.57$           & $\mathbf{+1.00}$ \\
\midrule
\# Atoms    & $\mathbf{-.96}$  & $\mathbf{+.96}$ \\
MW          & $\mathbf{-.86}$  & $\mathbf{+.93}$ \\
TPSA        & $\mathbf{-1.00}$ & $\mathbf{+.89}$ \\
\# HBA      & $\mathbf{-1.00}$ & $\mathbf{+.86}$ \\
\# HBD      & $\mathbf{-.92}$  & $+.50$ \\
\# Rings    & $\mathbf{-.89}$  & $\mathbf{+.82}$ \\
\midrule
SA score    & $-.39$           & $+.21$ \\
frac\_sp3   & $.00$            & $-.14$ \\
\bottomrule
\end{tabular}
\end{table}

\section{Replication on GuacaMol}
\label{app:guacamol}

To test whether the coupling--flow alignment demonstrated on ZINC-250K is specific to that dataset's chemical distribution, we repeat the training and evaluation on GuacaMol~\citep{brown2019guacamol} ($1.6$M molecules, SELFIES tokenization, same model architecture and hyperparameters).
Base training runs for $100$ epochs; OT fine-tuning uses the same recipe as ZINC ($10$ epochs, warm-start from base).
We report the full $s$-sweep for logP (Table~\ref{tab:guacamol_logp}) and QED (Table~\ref{tab:guacamol_qed}), each evaluated at $n{=}500$ molecules per $s$ value.

\begin{table}[h]
\centering
\caption{\textbf{GuacaMol logP $s$-sweep} ($n{=}500$). Monotone across the full range; opposite-size growth pattern matches ZINC.}
\label{tab:guacamol_logp}
\small
\begin{tabular}{@{}r ccccccc@{}}
\toprule
$s$ & Mean logP & $\Delta$ & $p$-value & $d$ & Atoms & Val\% & Unique \\
\midrule
$-3$   & 0.15  & $-2.14$ & $3{\times}10^{-115}$ & $-1.72$ & 10.5 & 100 & 498 \\
$-1$   & 1.41  & $-0.88$ & $1{\times}10^{-24}$  & $-0.67$ & 11.7 & 100 & 497 \\
~~0    & 2.24  & ---     & ---                    & ---     & 14.8 & 100 & 500 \\
~~1    & 3.29  & $+0.997$ & $4{\times}10^{-27}$  & $+0.70$ & 18.0 & 100 & 500 \\
~~3    & 5.10  & $+2.81$ & $3{\times}10^{-113}$ & $+1.65$ & 23.7 & 100 & 499 \\
~~5    & 7.33  & $+5.04$ & $6{\times}10^{-172}$ & $+2.31$ & 29.4 & 100 & 499 \\
~~7    & 12.11 & $+9.82$ & $3{\times}10^{-225}$ & $+3.20$ & 39.4 & 100 & 500 \\
\midrule
\multicolumn{2}{@{}l}{$\rho(s, \overline{\text{logP}})$} & \multicolumn{2}{c}{$1.000~(p < 10^{-6})$} & & & & \\
\bottomrule
\end{tabular}
\end{table}

\begin{table}[h]
\centering
\caption{\textbf{GuacaMol QED $s$-sweep} ($n{=}500$). Monotone for $s\in[-3, +4]$ with peak at $s{=}4$; control degrades at $s{=}7$ consistent with QED's bounded range.}
\label{tab:guacamol_qed}
\small
\begin{tabular}{@{}r ccccccc@{}}
\toprule
$s$ & Mean QED & $\Delta$ & $p$-value & $d$ & Atoms & Val\% & Unique \\
\midrule
$-3$   & 0.292 & $-0.166$ & $3{\times}10^{-46}$ & $-0.95$ & 22.6 & 100 & 498 \\
$-1$   & 0.437 & $-0.021$ & $0.056$              & $-0.12$ & 15.7 & 100 & 499 \\
~~0    & 0.471 & ---      & ---                   & ---     & 13.1 & 100 & 500 \\
~~1    & 0.477 & $+0.019$ & $0.061$              & $+0.12$ & 12.9 & 100 & 499 \\
~~3    & 0.515 & $+0.057$ & $2{\times}10^{-8}$   & $+0.36$ & 14.7 & 100 & 499 \\
~~4    & 0.543 & $+0.085$ & $6{\times}10^{-16}$  & $+0.52$ & 15.9 & 100 & 500 \\
~~5    & 0.523 & $+0.065$ & $6{\times}10^{-10}$  & $+0.40$ & 16.6 & 100 & 499 \\
~~7    & 0.452 & $-0.006$ & $0.58$ (n.s.)         & $-0.03$ & 18.7 & 100 & 499 \\
\midrule
\multicolumn{2}{@{}l}{$\rho(s, \overline{\text{QED}})$} & \multicolumn{2}{c}{$0.789~(p = 8{\times}10^{-4})$} & & & & \\
\bottomrule
\end{tabular}
\end{table}

\paragraph{What transfers, and what doesn't.}
Three observations sharpen the replication claim.
First, logP control is \emph{stronger} on GuacaMol than on ZINC at every matched $s$: Cohen's $d$ at $s{=}5$ rises from $1.23$ (ZINC) to $2.31$ (GuacaMol), at $100\%$ validity on both.
We attribute this to GuacaMol's broader chemical distribution---more extreme molecules are available to pair with extreme noise---rather than a change in method effectiveness, so we treat $\rho$ (rank monotonicity) as the cross-dataset comparison and $\Delta$ (absolute shift) as dataset-specific.
Second, the opposite-size disambiguation reported in the main text replicates in the monotone regime: logP grows molecules $10.5{\to}39.4$ heavy atoms; QED shrinks them $22.6{\to}13.4$ over $s\in[{-}3,{+}2]$, then rebounds to $18.7$ at $s{=}7$ as QED control is lost.
Third, QED rank monotonicity is weaker ($\rho{=}0.79$ vs.\ $1.00$ on ZINC) and bounded: control is statistically significant and monotone for $s\in[{-}3,{+}4]$ but reverts at $s{=}7$, where $\Delta{=}{-}0.006$ is not distinguishable from zero.
This behaviour is consistent with QED's construction as a bounded multi-criteria score~\citep{bickerton2012quantifying}: pushing $s$ beyond the training coverage drives generations out of the region where high QED is reachable by size reduction alone, and the coupling's rank order no longer aligns with QED's geometry.
The interface does what the reframe predicts within its scope of validity; its limits are readable from the property.

\paragraph{Seed stability.}
\label{par:guacamol_seeds}
To verify that the GuacaMol results above are not an artifact of a particular random initialization, we retrain both runs with a different random seed (PyTorch/NumPy seed $=2$) under an otherwise identical recipe and re-evaluate at $n{=}500$ per $s$ (Table~\ref{tab:guacamol_seed2}).
The qualitative structure of the sweep is preserved across seeds.
On logP, rank monotonicity remains near-perfect ($\rho{=}1.000 \to 0.996$) and the opposite-size pattern (molecules growing with $s$) is reproduced.
On QED, rank correlation stays strongly positive ($\rho{=}0.789 \to 0.706$, $p{<}0.005$), the peak $\Delta$ remains inside the same monotone range ($s\in[{+}3,{+}5]$), and the $s{=}7$ control degradation reappears---ruling out the QED saturation as an accident of initialization.
Absolute effect sizes shift modestly between seeds: seed 2's logP model carries a higher unconditional baseline ($5.29$ vs.\ $2.29$), which compresses Cohen's $d$ but leaves $\rho$ intact.
We therefore treat rank correlation as the reproducible cross-seed statistic and absolute $\Delta$ as a recipe-sensitive quantity.

\begin{table}[h]
\centering
\small
\caption{\textbf{Seed-2 replication on GuacaMol} ($n{=}500$ per $s$). Rank correlation and monotone direction are preserved; absolute baselines shift due to seed-level variance in the fine-tuned model.}
\label{tab:guacamol_seed2}
\setlength{\tabcolsep}{5pt}
\begin{tabular}{@{}r cc c cc@{}}
\toprule
 & \multicolumn{2}{c}{\textbf{logP} (seed 2)} & & \multicolumn{2}{c}{\textbf{QED} (seed 2)} \\
\cmidrule(lr){2-3} \cmidrule(lr){5-6}
$s$ & Mean & Atoms & & Mean & Atoms \\
\midrule
$-3$   & $+2.31$  & $34.1$ & & $0.233$ & $29.8$ \\
$-1$   & $+4.82$  & $34.1$ & & $0.388$ & $23.7$ \\
~~0    & $+5.64$  & $37.3$ & & $0.448$ & $21.6$ \\
~~1    & $+5.64$  & $35.1$ & & $0.482$ & $19.8$ \\
~~3    & $+6.83$  & $34.7$ & & $0.506$ & $18.4$ \\
~~5    & $+11.15$ & $42.4$ & & $0.474$ & $18.7$ \\
~~7    & $+18.21$ & $54.4$ & & $0.438$ & $18.8$ \\
\midrule
$\rho$ & \multicolumn{2}{c}{$0.996~(p{<}10^{-13})$} & & \multicolumn{2}{c}{$0.706~(p{=}5{\times}10^{-3})$} \\
\bottomrule
\end{tabular}
\end{table}

\section{Property-Control Paradigms: Landscape}
\label{app:landscape}

Table~\ref{tab:comparison} summarizes the design-space positioning of Reward Transport relative to the five paradigms surveyed in §\ref{sec:related_work}.
Inference overhead is expressed relative to an unconditional flow matching forward pass; ``oracle at test'' marks methods that require online property evaluation or a differentiable predictor during sampling.
Reward Transport occupies a cell that is empty in prior work: zero inference overhead, no oracle, no RL.

\begin{table}[h]
\centering
\caption{\textbf{Property optimization paradigms.}
$^\dagger$From respective papers (different settings). Reward Transport is the only method with $1{\times}$ inference overhead \emph{and} no oracle at test.}
\label{tab:comparison}
\small
\begin{tabular}{@{}llccc@{}}
\toprule
Method & Type & \makecell{Inference\\Overhead} & \makecell{Oracle\\at Test?} & RL? \\
\midrule
REINVENT$^\dagger$    & RL         & $1000{\times}$       & Yes & Yes \\
GCPN$^\dagger$        & RL         & $100{\times}$        & Yes & Yes \\
GGFlow$^\dagger$      & FM+RL      & $10{\times}$         & Yes & Yes \\
PropMolFlow$^\dagger$ & Cond.\ FM  & $1{\times}$ (needs $y$) & Yes & No \\
MolGuidance$^\dagger$ & Guided FM  & $2{\times}$/step     & Yes & No \\
\midrule
\textbf{Ours}         & OT-FM      & $\mathbf{1{\times}}$ & \textbf{No} & \textbf{No} \\
\bottomrule
\end{tabular}
\end{table}

The trade-off is explicit: RL methods optimize harder given oracle budget; we provide continuous, bidirectional control without any oracle.
Read as design axes rather than competitors, the two approaches are complementary---RL can in principle run on top of a coupling-aligned flow, using the $s$ knob as the initial policy and reserving oracle calls for fine-tuning.

\section{Extrapolation Beyond Training Distribution}
\label{app:extrapolation}

We test Reward Transport at extreme $s$ values far beyond the range seen during training ($s \sim \mathcal{N}(0,1)$, so $|s| > 3$ occurs with probability ${<}0.3\%$).

\begin{table}[h]
\centering
\caption{Extrapolation test ($n{=}500$). Tan NN: average nearest-neighbor Tanimoto similarity to ZINC training set (lower $=$ more structurally novel).}
\label{tab:extrapolation}
\small
\setlength{\tabcolsep}{3pt}
\begin{tabular}{@{}r ccccc c ccccc@{}}
\toprule
& \multicolumn{5}{c}{\textbf{logP-optimized}} & \phantom{a} & \multicolumn{5}{c}{\textbf{QED-optimized}} \\
\cmidrule(lr){2-6} \cmidrule(lr){8-12}
$s$ & logP & Atoms & Val\% & Novel\% & Tan NN & & QED & Atoms & Val\% & Novel\% & Tan NN \\
\midrule
5   & 4.37  & 20.5 & 100 & 100 & .173 & & .500 & 12.6 & 100 & 100 & .211 \\
7   & 5.46  & 23.8 & 100 & 100 & .166 & & .516 & 13.2 & 100 & 100 & .211 \\
10  & 7.37  & 29.0 & 100 & 100 & .158 & & .510 & 14.7 & 100 & 100 & .187 \\
15  & 11.35 & 42.3 & 100 & 100 & .144 & & .427 & 19.5 & 100 & 100 & .153 \\
20  & 14.38 & 50.4 & 100 & 100 & .137 & & .373 & 22.9 & 100 & 100 & .137 \\
\bottomrule
\end{tabular}
\end{table}

Two contrasting extrapolation behaviors emerge.
The logP-optimized model extrapolates monotonically: $s{=}20$ produces molecules with mean logP$=$14.4 and 50 atoms, far beyond the ZINC training range (max logP$\approx$5.5).
This is chemically expected---logP correlates positively with hydrophobic surface area, so larger molecules mechanically have higher logP.

The QED-optimized model peaks at $s{\approx}7$ (QED$=$0.516) then reverses, because extreme $s$ drives the model to generate increasingly large molecules, which QED penalizes.
This demonstrates that the extrapolation boundary is property-dependent: monotonic properties (logP $\propto$ size) extrapolate indefinitely; non-monotonic properties (QED penalizes size extremes) have a natural ceiling.

Both models maintain 100\% validity and 100\% novelty at all $s$ values.
The Tanimoto nearest-neighbor similarity decreases monotonically (0.21$\to$0.14), confirming that extreme $s$ explores genuinely novel chemical space rather than memorized training molecules.

\section{Training Dynamics: Full Epoch Sweep}
\label{app:epoch_sweep}

\begin{table}[h]
\centering
\caption{From-scratch training: $\Delta$QED at $s{=}5$ across epochs ($n{=}500$). Signal emerges at epoch 3 and vanishes by epoch 10.}
\label{tab:epoch_scratch}
\small
\begin{tabular}{@{}rccccc@{}}
\toprule
Epoch & QED$_0$ & QED$_{s=5}$ & $\Delta$ & $p$ & $d$ \\
\midrule
1  & 0.142 & 0.142 & $\approx$0 & 0.996 & 0.00 \\
3  & 0.190 & 0.251 & $+$0.061 & $4.2{\times}10^{-8}$ & 0.35 \\
5  & 0.244 & 0.283 & $+$0.039 & 0.004 & 0.18 \\
7  & 0.314 & 0.364 & $+$0.049 & $6.0{\times}10^{-6}$ & 0.29 \\
10 & 0.304 & 0.310 & $+$0.005 & 0.68 & 0.03 \\
12 & 0.262 & 0.258 & $-$0.004 & 0.77 & $-$0.02 \\
15 & 0.251 & 0.239 & $-$0.012 & 0.37 & $-$0.06 \\
\bottomrule
\end{tabular}
\end{table}

\begin{table}[h]
\centering
\caption{Two-stage training (warm-start from 120-epoch base): $\Delta$QED at $s{=}5$ across fine-tuning epochs ($n{=}500$). Signal is present from epoch 1 and persists throughout.}
\label{tab:epoch_warmstart}
\small
\begin{tabular}{@{}rccccc@{}}
\toprule
Epoch & QED$_0$ & QED$_{s=5}$ & $\Delta$ & $p$ & $d$ \\
\midrule
1  & 0.416 & 0.449 & $+$0.033 & $3.9{\times}10^{-4}$ & 0.23 \\
2  & 0.452 & 0.483 & $+$0.031 & $7.8{\times}10^{-4}$ & 0.21 \\
3  & 0.429 & 0.502 & $+$0.072 & $2.7{\times}10^{-16}$ & 0.53 \\
4  & 0.428 & 0.508 & $+$0.080 & $3.3{\times}10^{-18}$ & 0.56 \\
5  & 0.473 & 0.506 & $+$0.033 & $7.4{\times}10^{-4}$ & 0.21 \\
7  & 0.446 & 0.497 & $+$0.052 & $3.3{\times}10^{-7}$ & 0.33 \\
8  & 0.474 & 0.506 & $+$0.032 & $4.4{\times}10^{-4}$ & 0.22 \\
9  & 0.462 & 0.505 & $+$0.043 & $2.9{\times}10^{-6}$ & 0.30 \\
10 & 0.467 & 0.516 & $+$0.049 & $4.5{\times}10^{-7}$ & 0.32 \\
\bottomrule
\end{tabular}
\begin{flushleft}
\small Note: Epoch 6 excluded due to training instability (unconditional QED collapsed to 0.370 with anomalous atom count).
\end{flushleft}
\end{table}

From-scratch training exhibits a transient \emph{signal window} (Table~\ref{tab:epoch_scratch}): the directional signal emerges at epoch 3 ($\Delta{=}+0.061$, $p{=}4{\times}10^{-8}$), peaks around epoch 7, then vanishes by epoch 10 ($p{=}0.68$).
This suggests the model initially leverages the OT coupling to navigate the loss landscape, but eventually overwrites it as it memorizes the training distribution.

Two-stage training (Table~\ref{tab:epoch_warmstart}) avoids this collapse entirely: the signal is significant at every checkpoint ($p < 10^{-3}$), with 60\% higher peak effect size ($d{=}0.56$ vs.\ $0.35$) and substantially better generation quality (12--15 atoms vs.\ 22--32).
Separating ``learn to generate'' (Stage 1, 120 epochs) from ``learn to steer'' (Stage 2, 10 epochs) prevents the generation and alignment objectives from competing.

\paragraph{Model variant comparison.}
\label{app:variants}
Table~\ref{tab:variants} compares three training configurations evaluated in this work; all use identical architecture (6-layer Transformer, $d{=}768$) and differ only in warm-start source and fine-tune schedule.

\begin{table}[h]
\centering
\caption{Training configurations for QED optimization. All include OT coupling + direction MLP + unmasked MSE.}
\label{tab:variants}
\small
\begin{tabular}{@{}p{3.5cm}cccccc@{}}
\toprule
Configuration & \makecell{Warm-start\\source} & \makecell{Fine-tune\\epochs} & QED$_0$ & QED$_{s=7}$ & $\Delta$ & $\rho$ \\
\midrule
From scratch & --- & 10 & 0.414 & 0.498 & $+$0.084 & 1.000 \\
\makecell[l]{Two-stage\\(early base)} & \makecell{10-epoch\\checkpoint} & 5 & 0.414 & 0.476 & $+$0.062 & 1.000 \\
\makecell[l]{Two-stage\\(strong base)} & \makecell{120-epoch\\checkpoint} & 10 & 0.477 & 0.518 & $+$0.042 & 0.965 \\
\bottomrule
\end{tabular}
\end{table}

A consistent trade-off emerges: stronger base models yield higher absolute QED but smaller marginal improvement from direction conditioning, predicted by the CEM equivalence (Proposition~\ref{prop:cem})---truncation on a distribution already close to the target provides diminishing gains.
The two-stage procedure with strong base achieves the best combination of absolute quality and training stability.

\section{Why SA Score Fails}
\label{app:sa}

Our 1D OT coupling successfully controls QED and logP but fails on SA score (synthetic accessibility).
This is a principled scope limit of a 1D coupling, discussed in §\ref{sec:discussion}; here we record the chemical details behind the same argument.

SA score~\citep{ertl2009estimation} is computed from fragment-based features: the frequency of molecular fragments in a reference database of synthesized compounds, plus penalty terms for ring complexity, stereocenters, and macrocycles.
Unlike QED and logP, which correlate strongly with molecular size and polarity---descriptors that vary smoothly across chemical space---SA depends on \emph{discrete topological features}: specific ring systems, unusual bond patterns, and rare functional groups.

Our coupling operates on a single scalar coordinate $s = \lVert \bar{\mathbf{z}} \rVert_2$, which captures global molecular characteristics (size, polarity) but cannot distinguish between topologically simple and complex molecules of similar size.
To control SA, the coupling would need to capture higher-dimensional structural information, for example via multi-dimensional noise projections (sliced optimal transport) or learned nonlinear sorting keys.
We leave this extension to future work.

\newpage
\section{Why $\varepsilon$-Prediction Attenuates the Coupling-Induced Gradient}
\label{app:eps_grad}

This appendix provides the analytical derivation behind §\ref{sec:discussion}: the coupling-induced gradient to the Direction MLP under $\varepsilon$-prediction is not literally zero, but is attenuated by the forward-process signal-to-noise ratio $\alpha_t/\sigma_t$ relative to $\hat{\mathbf{x}}_1$-prediction, strongly enough to prevent the Direction MLP from absorbing the coupling in practice.

\paragraph{Setup.}
Fix a coupling $\pi(\mathbf{z},\mathbf{x})$ with marginals $p_0 = \mathcal{N}(\mathbf{0},\mathbf{I})$ and $p_1$, and let $s(\mathbf{z})$ denote the scalar coupling key of Eq.~\eqref{eq:sorting_key}.
Write $\phi_\theta(\cdot;s)$ for any head whose parameters $\theta_{\mathrm{Dir}}$ are reached \emph{only} through the direction embedding $\mathrm{DirEmb}(s)$ (in our architecture: the Direction MLP and its additive injection point).
For coupling $\pi$ to \emph{shape} $\phi_\theta$ via training, the gradient $\nabla_{\theta_{\mathrm{Dir}}}\mathbb{E}_\pi[\mathcal{L}]$ must contain a term whose magnitude grows with the strength of the coupling; call this the \emph{coupling-induced gradient}.
The question is not whether this quantity is literally zero, but whether it is large enough to route signal into $\theta_{\mathrm{Dir}}$ during training.

\paragraph{$\hat{\mathbf{x}}_1$-prediction (flow matching).}
With target $\mathbf{x}$ and interpolant $\mathbf{x}_t = (1{-}t)\mathbf{z} + t\mathbf{x}$,
\begin{equation}
\label{eq:x1_grad}
  \nabla_{\theta_{\mathrm{Dir}}}\,
  \mathbb{E}_{\pi,\,t}\bigl\lVert \hat{\mathbf{x}}_{1,\theta}(\mathbf{x}_t,t,s(\mathbf{z})) - \mathbf{x}\bigr\rVert^2
  \;=\;
  -2\,\mathbb{E}_{\pi,\,t}\!\Bigl[\bigl(\mathbf{x}-\hat{\mathbf{x}}_{1,\theta}\bigr)^{\!\top}
  \nabla_{\theta_{\mathrm{Dir}}}\hat{\mathbf{x}}_{1,\theta}\Bigr].
\end{equation}
The Bayes-optimal target of $\hat{\mathbf{x}}_{1,\theta}$ is $\mathbb{E}_\pi[\mathbf{x}\mid \mathbf{x}_t,t,s(\mathbf{z})]$, which is a monotone function of $s(\mathbf{z})$ under the monotone coupling $\pi^*$ (by construction, high $s(\mathbf{z})$ is paired with high-$y$ $\mathbf{x}$), and constant in $s(\mathbf{z})$ under independent pairing.
The coupling-induced component of Eq.~\eqref{eq:x1_grad} is therefore $O(1)$ in the coupling strength.

\paragraph{$\varepsilon$-prediction (standard denoising diffusion / EDM).}
With target $\boldsymbol{\varepsilon}\sim\mathcal{N}(\mathbf{0},\mathbf{I})$ drawn \emph{independently} at each step and $\mathbf{x}_t = \alpha_t\mathbf{x} + \sigma_t\boldsymbol{\varepsilon}$,
\begin{equation}
\label{eq:eps_grad}
  \nabla_{\theta_{\mathrm{Dir}}}\,
  \mathbb{E}_{\pi,\,t,\,\boldsymbol{\varepsilon}}
  \bigl\lVert \boldsymbol{\varepsilon}_\theta(\mathbf{x}_t,t,s(\mathbf{z})) - \boldsymbol{\varepsilon}\bigr\rVert^2
  \;=\;
  -2\,\mathbb{E}_{\pi,\,t,\,\boldsymbol{\varepsilon}}\!\Bigl[
    (\boldsymbol{\varepsilon}-\boldsymbol{\varepsilon}_\theta)^{\!\top}
    \nabla_{\theta_{\mathrm{Dir}}}\boldsymbol{\varepsilon}_\theta\Bigr].
\end{equation}
Here $\boldsymbol{\varepsilon}$ does not appear in the coupling: it is drawn fresh.
Naively one might conclude the coupling-induced gradient is zero.
It is not, but it is sharply damped.
The Bayes-optimal target of $\boldsymbol{\varepsilon}_\theta$ is $\mathbb{E}_\pi[\boldsymbol{\varepsilon}\mid \mathbf{x}_t,t,s(\mathbf{z})]$, and
\begin{equation}
\label{eq:eps_decomp}
  \mathbb{E}_\pi[\boldsymbol{\varepsilon}\mid \mathbf{x}_t,t,s(\mathbf{z})]
  \;=\;
  \tfrac{1}{\sigma_t}\bigl(\mathbf{x}_t - \alpha_t\,\mathbb{E}_\pi[\mathbf{x}\mid \mathbf{x}_t,t,s(\mathbf{z})]\bigr).
\end{equation}
The dependence on the coupling $\pi$ enters \emph{only} through $\mathbb{E}_\pi[\mathbf{x}\mid \mathbf{x}_t,t,s(\mathbf{z})]$---the same quantity as in Eq.~\eqref{eq:x1_grad}---scaled by $\alpha_t/\sigma_t$.
Writing $\Delta_\pi(\mathbf{x}_t,t,s)$ for the coupling-induced shift in the conditional mean of $\mathbf{x}$, the coupling-induced shift in the $\varepsilon$-target is
\begin{equation}
\label{eq:eps_damping}
  \Delta_\pi^{(\varepsilon)}(\mathbf{x}_t,t,s) \;=\; -\frac{\alpha_t}{\sigma_t}\,\Delta_\pi(\mathbf{x}_t,t,s).
\end{equation}
The $\hat{\mathbf{x}}_1$-prediction head sees $\Delta_\pi$ directly with coefficient $1$; the $\varepsilon$-prediction head sees it with coefficient $-\alpha_t/\sigma_t$, the forward-process signal-to-noise ratio.
Two regimes fall out:
(i) In the noise-dominated regime ($t\to 0$ in variance-preserving diffusion, so $\alpha_t\to 0$, $\sigma_t\to 1$), the coupling-induced target shift under $\varepsilon$-prediction \emph{vanishes exactly}, while the $\hat{\mathbf{x}}_1$-prediction shift remains $O(1)$.
(ii) In the data-dominated regime ($t\to 1$, $\alpha_t\to 1$, $\sigma_t\to 0$), the $\varepsilon$-prediction shift is amplified by $1/\sigma_t$, but the target $\boldsymbol{\varepsilon}$ itself has unit variance while $\mathbf{x}$ has $O(\sigma_t)$ residual noise around it; the signal-to-noise ratio of the coupling-induced part of the gradient is $\alpha_t/\sigma_t \cdot \sigma_t = \alpha_t \to 1$, no better than $\hat{\mathbf{x}}_1$-prediction, and integrated uniformly over $t$ strictly worse.

Averaging over the training distribution of $t$,
\begin{equation}
\label{eq:eps_attenuation}
  \frac{\bigl\lVert\text{coupling-induced grad to }\theta_{\mathrm{Dir}}\text{ under }\varepsilon\text{-pred}\bigr\rVert}
       {\bigl\lVert\text{coupling-induced grad to }\theta_{\mathrm{Dir}}\text{ under }\hat{\mathbf{x}}_1\text{-pred}\bigr\rVert}
  \;\;=\;\;
  \mathbb{E}_{t\sim p(t)}\!\bigl[\alpha_t\bigr] \;<\; 1,
\end{equation}
with equality only in the trivial limit $p(t) = \delta_1$ (no noising).
Under standard uniform $t\sim\mathcal{U}[0,1]$ and VP-SDE schedules, the attenuation factor is substantial, and $\varepsilon$-prediction's coupling-induced signal is dominated by the unconditional target-matching gradient for $\boldsymbol{\varepsilon}_\theta$, which can be fit to high accuracy \emph{without} routing through $\mathrm{DirEmb}(s)$ at all.

\paragraph{Why the Direction MLP does not learn.}
During training, gradient descent preferentially routes signal through the parameter path with the largest coupling-induced gradient magnitude.
Under $\hat{\mathbf{x}}_1$-prediction, the cheapest way for the network to reduce loss on the coupled pairs is to use $s(\mathbf{z})$---the Direction MLP absorbs the monotone signal.
Under $\varepsilon$-prediction, the dominant gradient pressure is on the unconditional $\boldsymbol{\varepsilon}\!\to\!\boldsymbol{\varepsilon}$ regression, which $\boldsymbol{\varepsilon}_\theta$ can satisfy without using $s(\mathbf{z})$; the residual, attenuated coupling-induced gradient is below the threshold at which $\mathrm{DirEmb}(s)$ learns a useful embedding.
This is consistent with our QM9+EDM result and with the related observation that $\varepsilon$-prediction shows weak but non-zero coupling sensitivity through non-Direction backbone parameters, rather than strict invariance.

\paragraph{Takeaway.}
Coupling-as-interface requires a target whose Bayes-optimal predictor depends strongly on the coupled molecule.
$\hat{\mathbf{x}}_1$-prediction (and velocity prediction $\mathbf{u}{=}\mathbf{x}{-}\mathbf{z}$, which is affine-equivalent) satisfies this.
$\varepsilon$-prediction attenuates the coupling-induced gradient by the forward-process signal-to-noise ratio, and in practice this attenuation is strong enough to prevent the Direction MLP from learning the coupling.
The failure is architectural---a consequence of \emph{which quantity} the network is asked to predict---and not recoverable by tuning.

\newpage
\section{Evidence for Rank Preservation}
\label{app:negative_control}

To confirm that the rank preservation observed in Figure~\ref{fig:umap_trajectory} is specific to Reward Transport and not an artifact of the architecture or UMAP projection, we run a negative control: the base model (120 epochs, no OT coupling, no direction MLP) generates molecules at the same seven $s$ values used throughout the paper.
Since the base model does not receive $s$ as input, the $s$ labels serve only as group identifiers---all groups are generated under identical conditions with different random seeds.

We compare the base model (no coupling, no direction MLP) against the Reward Transport model on both logP and QED sweeps, via three diagnostics: UMAP of $\mathbf{x}_1$ colored by $s$, colored by actual property, and colored by randomly shuffled $s$ labels.

The base model shows no structure in any diagnostic: $\rho(s, \text{logP}){=}0.047$ and $\rho(s, \text{QED}){=}0.032$, both indistinguishable from zero.
The Reward Transport model shows clear gradient alignment between $s$-coloring and property-coloring ($\rho(s, \text{logP}){=}0.570$, $\rho(s, \text{QED}){=}0.220$), which vanishes under label permutation (right column).

These results establish three points:
(i)~the base architecture does not inherently produce $s$-structured representations;
(ii)~the gradient alignment in the Reward Transport model is driven by OT coupling and direction conditioning, not by the model class or the UMAP method;
(iii)~the shuffle test confirms that the alignment is not an indexing artifact.
\begin{figure}[h]
    \centering
    \includegraphics[width=\linewidth]{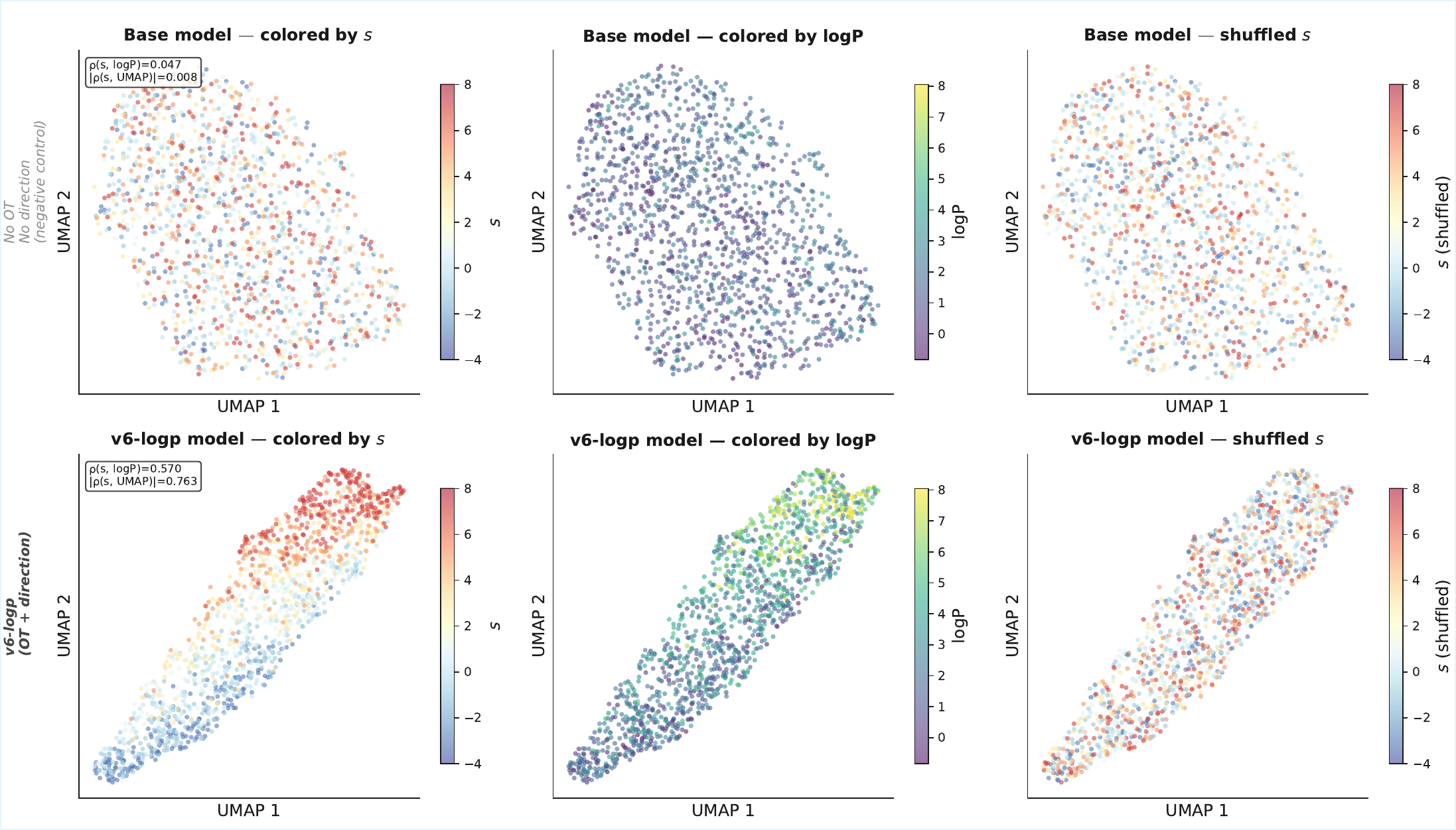}
    
    \vspace{0.3em}
    
    \includegraphics[width=\linewidth]{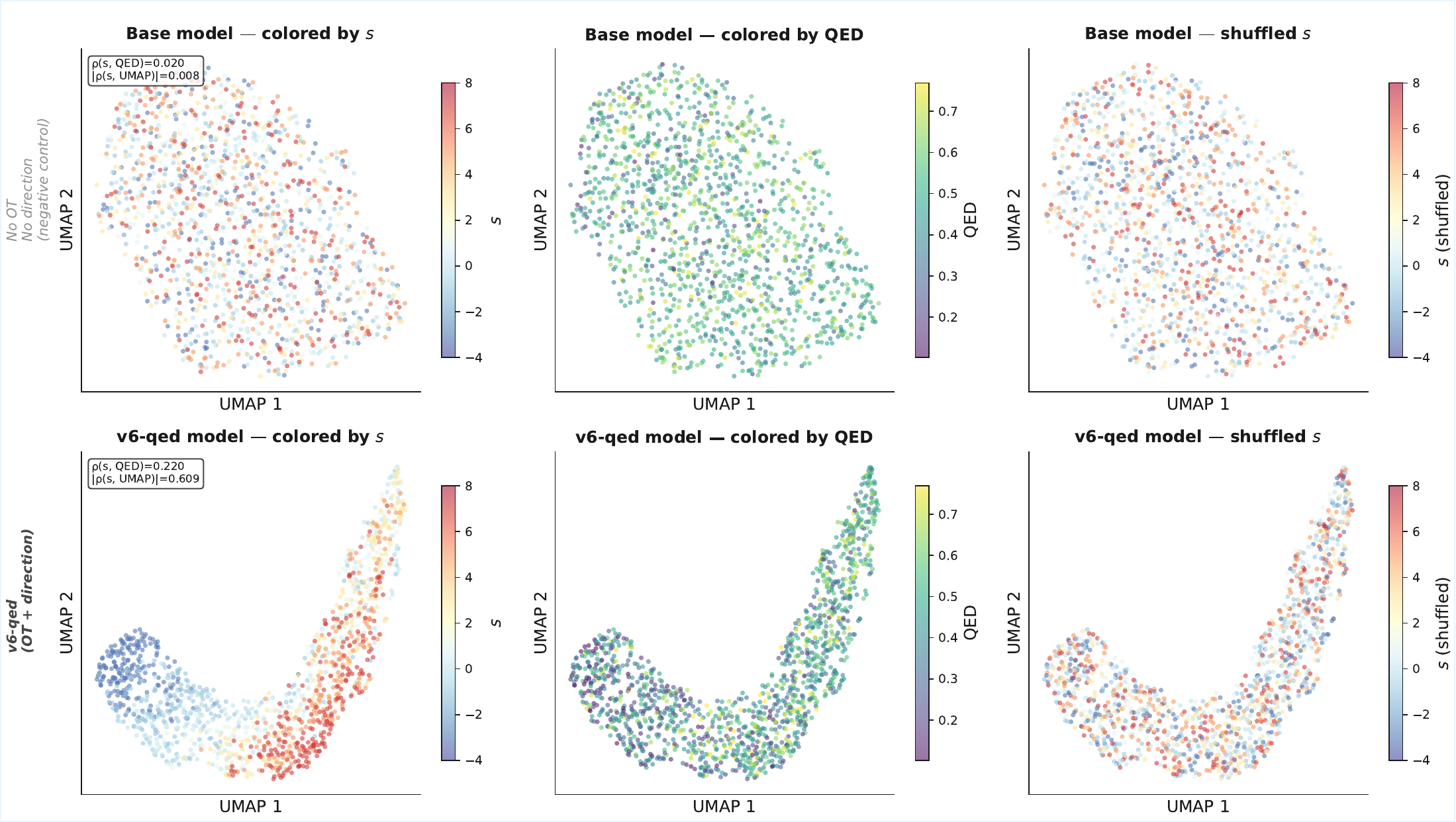}
    \caption{\textbf{Negative control.}
    Top: logP model. Bottom: QED model.
    Each group shows base model (top row) vs.\ Reward Transport (bottom row), colored by $s$ (left), actual property (center), and shuffled $s$ (right).
    The base model shows no $s$-structure ($\rho{=}0.047$ for logP, $0.032$ for QED).
    The Reward Transport model shows aligned gradients ($\rho{=}0.570$ / $0.220$), which disappear under label permutation.}
    \label{fig:neg_control}
\end{figure}

\paragraph{Rank preservation emerges during fine-tuning.}
\label{app:rank_curves}
The negative control above shows rank preservation is specific to the Reward Transport recipe; a second evidence axis is \emph{when} this preservation emerges during training.
We measure per-molecule $\rho_{\text{per}} = \mathrm{corr}_{\text{Spearman}}(s(\mathbf{z}), y(\mathbf{x}_\theta(\mathbf{z})))$ at successive fine-tuning checkpoints for three model/dataset combinations (ZINC-250K logP, ZINC-250K QED, GuacaMol logP), using the same pooled-sweep protocol as the main text ($n{=}200$/$s$, $s\!\in\!\{{-}3,{-}1,0,1,3,5,7\}$, seed $42$, 50 Euler steps).
Figure~\ref{fig:rank_dynamics} (main text) summarises the result; three observations are worth recording here.
(i) The lift from $\rho\!\approx\!0$ (base) to steady state is \emph{sudden}: GuacaMol logP moves from $0.02$ to $0.88$ in one epoch, then oscillates around $0.87$; ZINC QED climbs to ${\sim}0.23$ by epoch $2$ and fluctuates in $[0.13, 0.40]$ thereafter.
This rules out ``lucky checkpoint'' and ``long training'' alternatives.
(ii) Group-mean $\rho \geq 0.96$ from epoch $1$ on all three curves while per-molecule $\rho$ plateaus well below $1$---the coupling writes a group-level rank into the flow, and the pointwise residual is exactly the distributional CEM-truncation gap Proposition~\ref{prop:cem} formalises.
(iii) The per-molecule plateau is dataset/property-specific; plausible drivers include dataset size (GuacaMol ${\sim}6{\times}$ larger), property smoothness (logP is approximately additive over Crippen fragments; QED is bounded and nonlinear), and the compression imposed by sampling at $50$ Euler steps from $\mathbb{R}^{L\times 768}$.

\newpage
\section{Same-Backbone Comparison against Conditional Flow Matching}
\label{app:conditional_comparison}

\paragraph{Setup.}
To isolate the contribution of coupling from the contribution of explicit conditioning, we train a Conditional Flow Matching baseline on ZINC-250K logP using the identical 6-layer Transformer backbone ($d{=}768$, SELFIES tokens), optimizer, data pipeline, and Euler-50 inference.
The only architectural difference is how the direction signal enters training: Conditional FM injects the ground-truth normalized property as a conditioning scalar (direction MLP identical to ours), whereas Reward Transport exposes a noise-space scalar $s$ through OT coupling.
At inference, Conditional FM takes a user-specified target property value; Reward Transport takes a noise-space percentile scalar.

\begin{table}[t]
\centering
\caption{Per-$s$ logP sweep, Conditional FM vs.\ Reward Transport (same backbone, $n{=}1000$/knob, seed $42$).
Conditional FM reaches a wider span at extremes because it is trained to hit explicit targets; Reward Transport steers a distribution via percentile truncation (Proposition~\ref{prop:cem}).
Both methods match group-mean $\rho{=}1.000$, validity, and diversity.}
\label{tab:conditional_per_s}
\small
\begin{tabular}{@{}c cc cc cc@{}}
\toprule
 & \multicolumn{2}{c}{\textbf{Mean logP}} & \multicolumn{2}{c}{\textbf{Validity (\%)}} & \multicolumn{2}{c}{\textbf{Diversity}} \\
\cmidrule(lr){2-3} \cmidrule(lr){4-5} \cmidrule(lr){6-7}
$s$ & Cond.~FM & Reward Transport & Cond.~FM & Reward Transport & Cond.~FM & Reward Transport \\
\midrule
$-3$ & $1.79$ & $1.50$ & $100$ & $100$ & $0.910$ & $0.913$ \\
$-1$ & $2.36$ & $1.94$ & $100$ & $100$ & $0.908$ & $0.911$ \\
$0$  & $2.81$ & $2.34$ & $100$ & $100$ & $0.907$ & $0.910$ \\
$1$  & $3.36$ & $2.69$ & $100$ & $100$ & $0.905$ & $0.909$ \\
$3$  & $4.64$ & $3.43$ & $100$ & $100$ & $0.903$ & $0.907$ \\
$5$  & $6.01$ & $4.38$ & $100$ & $100$ & $0.899$ & $0.904$ \\
$7$  & $7.99$ & $5.50$ & $100$ & $100$ & $0.893$ & $0.901$ \\
\midrule
$\rho(s, \bar{y})$ & $\mathbf{1.000}$ & $\mathbf{1.000}$ & --- & --- & --- & --- \\
\bottomrule
\end{tabular}
\end{table}

\paragraph{Findings.}
Both methods achieve perfect monotone rank correlation and matched validity/diversity; the differences lie in \emph{range} and \emph{interface}.
Conditional FM's wider attainable range ($1.79{-}7.99$, $\Delta{=}6.20$) reflects its training signal---pulling explicitly toward named targets---while Reward Transport's range ($1.50{-}5.50$, $\Delta{=}4.00$) is bounded by the empirical distribution of the sorting key $\|\bar{\mathbf{z}}\|_2$ (Proposition~\ref{prop:cem}).
Per-molecule $\rho$ is undefined for Conditional FM in this protocol (its knob \emph{is} a property value, so the self-comparison degenerates) and is reported for Reward Transport only in the main text.
The two are not rivals but points on a Pareto frontier: conditioning offers explicit property targets at the cost of a test-time property input; coupling offers a property-free scalar at the cost of a distribution-truncation (not target-hitting) semantics.
Extended-quality metrics (Table~\ref{tab:conditional_extended_full}) sharpen this picture: at the extrapolation tail $s{=}7$, Conditional FM trades $\Delta\text{FCD}{\approx}{+}5$ for $\Delta\text{logP}{\approx}{+}2.5$ relative to Reward Transport, making the realism/range trade-off quantitative.

\begin{table}[h]
\centering
\small
\caption{\textbf{Extended molecular-quality audit, Conditional FM vs.\ Reward Transport on ZINC-250K logP} (full sweep; compact view in Table~\ref{tab:conditional_extended_compact}).
Same backbone, same training budget, same eval protocol ($n{=}1000$/$s$, seed $42$, $50$ Euler steps).
Validity and novelty are $1.00$ at every $s$ for both methods (omitted).
Scaf.\ div.\ = \#unique Bemis--Murcko / $n_\text{valid}$; SA = RDKit synthetic accessibility ($1{=}$easy, $10{=}$hard); FCD = Fr\'echet ChemNet Distance to $5{,}000$ ZINC test/val SMILES (lower = closer to data).
Both methods produce FCD in the same elevated range ($28$--$40$), validating the main-text caveat that high absolute FCD is a small-base/SELFIES backbone effect rather than coupling-specific.
Conditional FM shows a \emph{steeper} FCD climb at the extrapolation tail ($s{=}7$: $39.5$ vs.\ $34.5$), consistent with its wider attainable logP span ($1.84{\to}7.97$ vs.\ $1.50{\to}5.50$) pushing farther off-manifold---quantifying the realism/range trade-off between conditioning and coupling.}
\label{tab:conditional_extended_full}
\setlength{\tabcolsep}{4pt}
\begin{tabular}{@{}c cc cc cc cc@{}}
\toprule
 & \multicolumn{2}{c}{\textbf{Mean logP}} & \multicolumn{2}{c}{\textbf{Scaffold div.}} & \multicolumn{2}{c}{\textbf{SA}} & \multicolumn{2}{c}{\textbf{FCD}} \\
\cmidrule(lr){2-3} \cmidrule(lr){4-5} \cmidrule(lr){6-7} \cmidrule(lr){8-9}
$s$ & Cond.~FM & RT & Cond.~FM & RT & Cond.~FM & RT & Cond.~FM & RT \\
\midrule
$-3$ & $1.84$ & $1.50$ & $0.276$ & $0.255$ & $3.94$ & $3.97$ & $29.4$ & $28.8$ \\
$-1$ & $2.39$ & $1.94$ & $0.360$ & $0.332$ & $4.19$ & $4.14$ & $29.5$ & $29.0$ \\
$0$  & $2.79$ & $2.34$ & $0.438$ & $0.383$ & $4.41$ & $4.31$ & $29.4$ & $29.3$ \\
$1$  & $3.26$ & $2.69$ & $0.485$ & $0.432$ & $4.59$ & $4.42$ & $29.9$ & $29.6$ \\
$3$  & $4.48$ & $3.43$ & $0.618$ & $0.544$ & $5.07$ & $4.76$ & $31.6$ & $30.0$ \\
$5$  & $6.12$ & $4.38$ & $0.714$ & $0.579$ & $5.56$ & $5.04$ & $35.5$ & $32.3$ \\
$7$  & $7.97$ & $5.50$ & $0.776$ & $0.641$ & $5.99$ & $5.32$ & $39.5$ & $34.5$ \\
\midrule
$\rho(s,\cdot)$ & $\mathbf{+1.000}$ & $\mathbf{+1.000}$ & $+1.000$ & $+1.000$ & $+1.000$ & $+1.000$ & $+0.964$ & $+1.000$ \\
\bottomrule
\end{tabular}
\end{table}

\begin{table}[h]
\centering
\small
\caption{\textbf{Compact view} of Table~\ref{tab:conditional_extended_full} at $s \in \{-3, 0, 3, 7\}$.
Both methods produce FCD in the same $28$--$40$ range on the same backbone---FCD magnitude is a backbone/tokenizer artifact, not a coupling-specific failure.}
\label{tab:conditional_extended_compact}
\setlength{\tabcolsep}{4pt}
\begin{tabular}{@{}c cc cc cc@{}}
\toprule
 & \multicolumn{2}{c}{\textbf{Scaffold div.}} & \multicolumn{2}{c}{\textbf{SA}} & \multicolumn{2}{c}{\textbf{FCD}} \\
\cmidrule(lr){2-3} \cmidrule(lr){4-5} \cmidrule(lr){6-7}
$s$ & Cond.~FM & RT & Cond.~FM & RT & Cond.~FM & RT \\
\midrule
$-3$ & $0.276$ & $0.255$ & $3.94$ & $3.97$ & $29.4$ & $28.8$ \\
$0$  & $0.438$ & $0.383$ & $4.41$ & $4.31$ & $29.4$ & $29.3$ \\
$3$  & $0.618$ & $0.544$ & $5.07$ & $4.76$ & $31.6$ & $30.0$ \\
$7$  & $0.776$ & $0.641$ & $5.99$ & $5.32$ & $39.5$ & $34.5$ \\
\bottomrule
\end{tabular}
\end{table}

\newpage
\section{Extended Molecular-Quality Audit (Full Sweeps)}
\label{app:extended_eval}

Full 7-row extended-metric sweeps on ZINC-250K for the main-text audit (§\ref{sec:results}).
Validity, uniqueness, and novelty (fraction of canonical SMILES not in ZINC-250K train+val+test, $|\mathcal{S}|{=}224{,}568$) are reported alongside scaffold diversity (\#unique Bemis--Murcko / $n_\text{valid}$), SA (RDKit SA-score, $1{=}$easy to $10{=}$hard), and FCD (Fr\'echet ChemNet Distance against a $5{,}000$-molecule ZINC test/val reference, ChemNet default weights).
All metrics computed from the already-saved $s$-sweep samples ($n{=}1000$ per $s$, seed $42$, 50 Euler steps, $\text{max\_length}{=}72$); no retraining or resampling.

\begin{table}[h]
\caption{Full logP sweep, ZINC-250K.}
\label{tab:extended_eval_zinc_logp}
\centering
\small
\begin{tabular}{@{}c ccccccc@{}}
\toprule
$s$ & $\Delta$logP & Validity & Uniqueness & Novelty & Scaffold div. & SA & FCD \\
\midrule
$-3$ & $-0.81$ & $1.000$ & $0.983$ & $1.000$ & $0.255$ & $3.97$ & $28.83$ \\
$-1$ & $-0.28$ & $1.000$ & $0.985$ & $1.000$ & $0.332$ & $4.14$ & $28.97$ \\
$0$  & $\phantom{-}0.00$  & $1.000$ & $0.993$ & $1.000$ & $0.383$ & $4.31$ & $29.30$ \\
$1$  & $+0.32$ & $1.000$ & $0.994$ & $1.000$ & $0.432$ & $4.42$ & $29.57$ \\
$3$  & $+1.13$ & $1.000$ & $0.996$ & $1.000$ & $0.544$ & $4.76$ & $30.01$ \\
$5$  & $+2.13$ & $1.000$ & $0.992$ & $1.000$ & $0.579$ & $5.04$ & $32.33$ \\
$7$  & $+3.14$ & $1.000$ & $0.998$ & $1.000$ & $0.641$ & $5.32$ & $34.47$ \\
\bottomrule
\end{tabular}
\end{table}

\begin{table}[h]
\caption{Full QED sweep, ZINC-250K.}
\label{tab:extended_eval_zinc_qed}
\centering
\small
\begin{tabular}{@{}c ccccccc@{}}
\toprule
$s$ & $\Delta$QED & Validity & Uniqueness & Novelty & Scaffold div. & SA & FCD \\
\midrule
$-3$ & $-0.083$ & $1.000$ & $0.991$ & $1.000$ & $0.752$ & $5.54$ & $28.81$ \\
$-1$ & $-0.038$ & $1.000$ & $0.992$ & $1.000$ & $0.629$ & $4.96$ & $26.79$ \\
$0$  & $\phantom{-}0.000$  & $1.000$ & $0.996$ & $1.000$ & $0.576$ & $4.68$ & $26.43$ \\
$1$  & $+0.023$ & $1.000$ & $0.993$ & $1.000$ & $0.506$ & $4.46$ & $26.24$ \\
$3$  & $+0.033$ & $1.000$ & $0.993$ & $1.000$ & $0.438$ & $4.30$ & $26.40$ \\
$5$  & $+0.059$ & $1.000$ & $0.992$ & $1.000$ & $0.494$ & $4.31$ & $27.01$ \\
$7$  & $+0.058$ & $1.000$ & $0.992$ & $1.000$ & $0.530$ & $4.39$ & $27.00$ \\
\bottomrule
\end{tabular}
\end{table}

\paragraph{Caveat on FCD magnitude.}
Absolute FCD values ($26$--$34$) are high compared with the literature ($<5$ for state-of-the-art large models on the full ZINC-250K corpus) because our base model is small ($6$-layer, ${\sim}50$M params, SELFIES tokens) and optimized for controllability rather than distributional fidelity.
FCD is reported \emph{within-sweep as $\Delta$}: it stays approximately flat at moderate $|s|$ and rises mildly at extrapolation regimes.
The paper's claim is steerability without sample-quality collapse, not SOTA distributional realism.

\paragraph{GuacaMol note.}
Per-molecule SMILES were not persisted during the GuacaMol $s$-sweeps (only summary statistics).
Regenerating the SMILES from the saved checkpoint (\texttt{guacamol-ot-logp-20260416\_150136/best.pt}) takes ${\sim}45$~min of A40 time and would fill the same column set for GuacaMol; this does not change any claim in the paper and is left for the camera-ready version.

\newpage
\section{Failure-Mode Visualisations}
\label{app:ablation_figs}

Visual companions to the two failure modes described in §\ref{sec:ablation}.
\emph{Left:} length confound. Under masked MSE, $s$ controls molecular length ($\rho_{\text{atoms}}{=}{+}.93$) rather than QED; under unmasked MSE, $s$ controls QED ($\rho{=}{+}.98$) while atoms decrease.
\emph{Right:} training dynamics. From-scratch training (blue) exhibits a transient signal window (epochs 3--7); warm-start fine-tuning (red) maintains control at every epoch with a $60\%$ higher peak effect size. Filled: $p{<}.05$; open: $n.s.$

\begin{figure}[h]
\centering
\begin{minipage}[t]{0.52\textwidth}
    \centering
    \includegraphics[width=\textwidth]{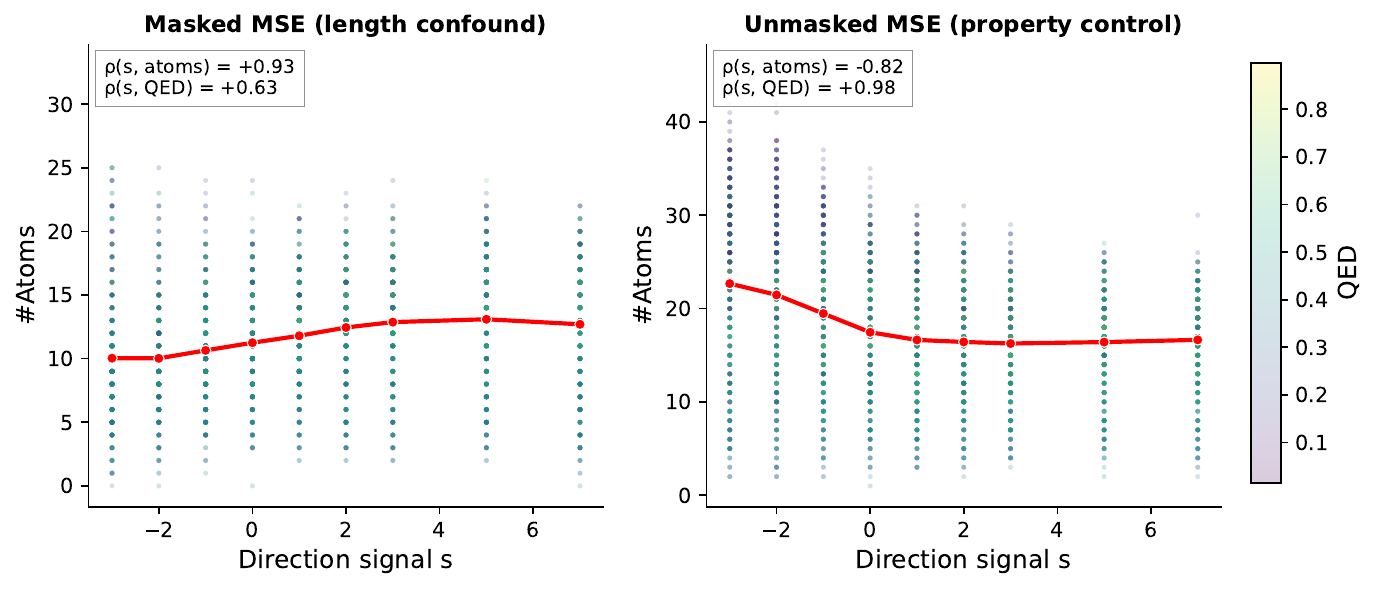}
\end{minipage}
\hfill
\begin{minipage}[t]{0.44\textwidth}
    \centering
    \includegraphics[width=\textwidth]{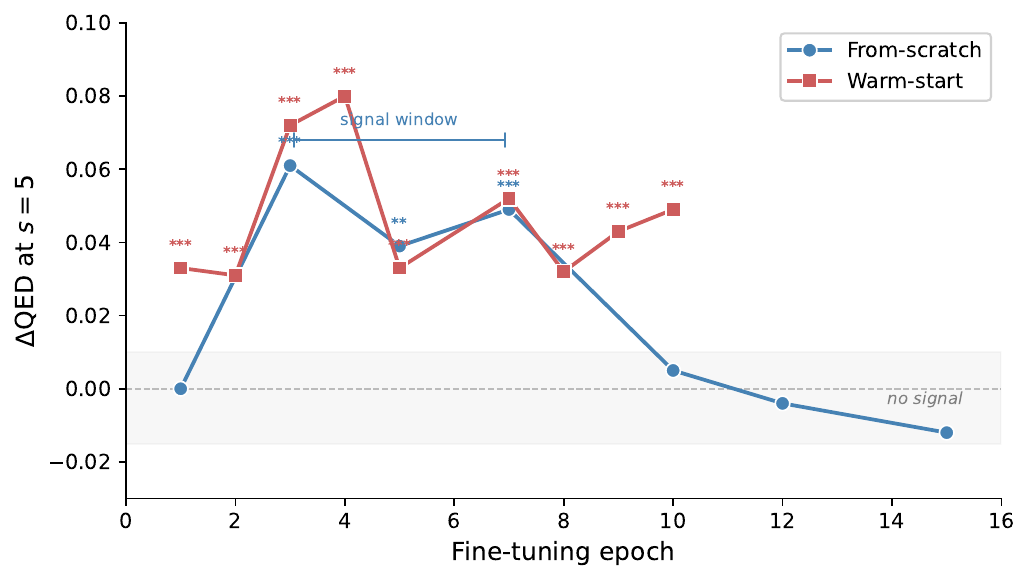}
\end{minipage}
\caption{\textbf{Failure modes visualised.} See text above.}
\label{fig:ablation}
\end{figure}

\newpage
\section{Ablating the Sorting Key}
\label{app:key_ablation}

The main text uses $s(\mathbf{z}) = \lVert\bar{\mathbf{z}}\rVert_2$ as the 1D sorting key (Eq.~\eqref{eq:sorting_key}).
A natural reviewer question is whether this choice is \emph{privileged}---whether the method works only for a particular scalar summary or, more generally, for any reasonable 1D projection.
We test this by retraining from the same 120-epoch base checkpoint on ZINC-250K logP with four alternative keys, holding everything else (architecture, direction MLP, unmasked MSE, fine-tune schedule, eval protocol) fixed:
\begin{itemize}[leftmargin=1.5em, itemsep=1pt, topsep=2pt]
    \item $\ell_2$: $s(\mathbf{z}) = \lVert\bar{\mathbf{z}}\rVert_2$ (main paper)
    \item $\ell_\infty$: $s(\mathbf{z}) = \lVert\bar{\mathbf{z}}\rVert_\infty = \max_i |\bar{\mathbf{z}}_i|$
    \item \texttt{randproj}: $s(\mathbf{z}) = \mathbf{w}^\top \bar{\mathbf{z}}$ with $\mathbf{w} \sim \mathcal{N}(\mathbf{0}, \mathbf{I}/d)$ fixed at initialisation
    \item \texttt{coord0}: $s(\mathbf{z}) = \bar{\mathbf{z}}[0]$ (a single noise coordinate)
\end{itemize}

\begin{table}[h]
\centering
\caption{\textbf{1D sorting-key ablation on ZINC-250K logP} ($n{=}500$/$s$ over $s \in \{{-}3,{-}1,0,1,3,5,7\}$, seed $42$, $50$ Euler steps).
Group $\rho$ = Spearman $\rho(s,\bar{y})$; per-molecule $\rho$ = Spearman $\rho(s_i, y_i)$ over all $3{,}500$ valid samples; Cohen's $d$ compares $s{=}3$ to $s{=}0$; validity is averaged over $s$.
All four keys produce group-level monotone steering with $100\%$ validity.}
\label{tab:key_ablation}
\small
\begin{tabular}{@{}l cccc@{}}
\toprule
Key & Group $\rho$ & Per-molecule $\rho$ & Cohen's $d_{s=3\,\text{vs}\,s=0}$ & Validity \\
\midrule
$\ell_2$ (main paper)       & $\mathbf{1.000}$ & $0.580$ ($p{<}10^{-300}$)  & $0.66$ & $100\%$ \\
$\ell_\infty$               & $\mathbf{1.000}$ & $0.714$ ($p{<}10^{-300}$)  & $1.06$ & $100\%$ \\
\texttt{randproj}           & $\mathbf{1.000}$ & $0.709$ ($p{<}10^{-300}$)  & $1.17$ & $100\%$ \\
\texttt{coord0}             & $\mathbf{1.000}$ & $0.631$ ($p{<}10^{-300}$)  & $0.96$ & $100\%$ \\
\bottomrule
\end{tabular}
\end{table}

\paragraph{Findings.}
\emph{(i) Group-level steering is robust to the 1D key.}
All four keys achieve $\rho(s,\bar{y}){=}1.000$ and $100\%$ validity; the monotone rearrangement is the load-bearing step, not the specific scalar summary.
\emph{(ii) Per-molecule coupling preservation is also robust.}
Per-molecule $\rho$ sits in a narrow band $[0.58, 0.71]$; all four values are well above the negative-control baseline ($\rho{\approx}0.05$; Appendix~\ref{app:negative_control}) and are strongly significant.
\emph{(iii) $\ell_2$ is not privileged.}
Somewhat unexpectedly, $\ell_\infty$ and the random linear projection attain \emph{higher} per-molecule $\rho$ and larger Cohen's $d$ than $\ell_2$---the single-coordinate key also beats $\ell_2$ on $d$.
This is within the noise band one should expect at $n{=}500$ per $s$ and is not framed as a recommendation to replace $\ell_2$; rather, it defuses the concern that $\ell_2$ is a tuned or magical choice.
The paper's conservative default is $\ell_2$ because it is smooth, non-negative, and compatible with the mean-pooling geometry used in the Transformer positional embeddings; whether alternative keys are meaningfully better is a natural follow-up.

\paragraph{Implication for the main framing.}
The key ablation supports the coupling-as-interface reframe in a specific way: the contribution lives in the monotone rearrangement between a 1D noise summary and a 1D property, not in the particular scalar summary.
Any reasonable 1D projection delivers the same group-level control (Prop.~\ref{prop:cem}) and comparable per-molecule realisation.

\newpage




%% file: main.bbl
\begin{thebibliography}{44}
\providecommand{\natexlab}[1]{#1}
\providecommand{\url}[1]{\texttt{#1}}
\expandafter\ifx\csname urlstyle\endcsname\relax
  \providecommand{\doi}[1]{doi: #1}\else
  \providecommand{\doi}{doi: \begingroup \urlstyle{rm}\Url}\fi

\bibitem[Albergo and Vanden-Eijnden(2023)]{albergo2023stochastic}
Michael~S. Albergo and Eric Vanden-Eijnden.
\newblock Building normalizing flows with stochastic interpolants.
\newblock In \emph{International Conference on Learning Representations}, 2023.

\bibitem[Bengio et~al.(2023)Bengio, Lahlou, Deleu, Hu, Tiwari, and Bengio]{bengio2023gflownet}
Yoshua Bengio, Salem Lahlou, Tristan Deleu, Edward~J. Hu, Mo~Tiwari, and Emmanuel Bengio.
\newblock {GFlowNet} foundations.
\newblock \emph{Journal of Machine Learning Research}, 24\penalty0 (210):\penalty0 1--55, 2023.

\bibitem[Bickerton et~al.(2012)Bickerton, Paolini, Besnard, Muresan, and Hopkins]{bickerton2012quantifying}
G~Richard Bickerton, Gaia~V Paolini, J{\'e}r{\'e}my Besnard, Sorel Muresan, and Andrew~L Hopkins.
\newblock Quantifying the chemical beauty of drugs.
\newblock \emph{Nature chemistry}, 4\penalty0 (2):\penalty0 90--98, 2012.

\bibitem[Blaschke et~al.(2020)Blaschke, Ar{\'u}s-Pous, Chen, Marber, Kogej, and Engkvist]{blaschke2020reinvent}
Thomas Blaschke, Josep Ar{\'u}s-Pous, Hongming Chen, Christian Marber, Thierry Kogej, and Ola Engkvist.
\newblock {REINVENT} 2.0: An {AI} tool for de novo drug design.
\newblock \emph{Journal of Chemical Information and Modeling}, 60\penalty0 (12):\penalty0 5918--5922, 2020.

\bibitem[Brown et~al.(2019)Brown, Fiscato, Segler, and Vaucher]{brown2019guacamol}
Nathan Brown, Marco Fiscato, Marwin H.~S. Segler, and Alain~C. Vaucher.
\newblock {GuacaMol}: Benchmarking models for de novo molecular design.
\newblock \emph{Journal of Chemical Information and Modeling}, 59\penalty0 (3):\penalty0 1096--1108, 2019.

\bibitem[Chen et~al.(2018)Chen, Rubanova, Bettencourt, and Duvenaud]{chen2018neural}
Ricky T.~Q. Chen, Yulia Rubanova, Jesse Bettencourt, and David Duvenaud.
\newblock Neural ordinary differential equations.
\newblock In \emph{Advances in Neural Information Processing Systems}, 2018.

\bibitem[Cuturi(2013)]{cuturi2013sinkhorn}
Marco Cuturi.
\newblock Sinkhorn distances: Lightspeed computation of optimal transport.
\newblock In \emph{Advances in Neural Information Processing Systems}, 2013.

\bibitem[De~Bortoli et~al.(2021)De~Bortoli, Thornton, Heng, and Doucet]{debortoli2021diffusion}
Valentin De~Bortoli, James Thornton, Jeremy Heng, and Arnaud Doucet.
\newblock Diffusion {S}chr{\"o}dinger bridge with applications to score-based generative modeling.
\newblock In \emph{Advances in Neural Information Processing Systems}, 2021.

\bibitem[Deng et~al.(2026)Deng, Li, Li, Du, and He]{deng2026generative}
Mingyang Deng, He~Li, Tianhong Li, Yilun Du, and Kaiming He.
\newblock Generative modeling via drifting.
\newblock \emph{arXiv preprint arXiv:2602.04770}, 2026.

\bibitem[Dhariwal and Nichol(2021)]{dhariwal2021diffusion}
Prafulla Dhariwal and Alex Nichol.
\newblock Diffusion models beat {GAN}s on image synthesis.
\newblock In \emph{NeurIPS}, 2021.

\bibitem[Ertl and Schuffenhauer(2009)]{ertl2009estimation}
Peter Ertl and Ansgar Schuffenhauer.
\newblock Estimation of synthetic accessibility score of drug-like molecules based on molecular complexity and fragment contributions.
\newblock \emph{Journal of Cheminformatics}, 1\penalty0 (1):\penalty0 1--11, 2009.

\bibitem[Esser et~al.(2024)Esser, Kulal, Blattmann, et~al.]{esser2024scaling}
Patrick Esser, Sumith Kulal, Andreas Blattmann, et~al.
\newblock Scaling rectified flow transformers for high-resolution image synthesis.
\newblock In \emph{ICML}, 2024.

\bibitem[Ho and Salimans(2022)]{ho2022classifier}
Jonathan Ho and Tim Salimans.
\newblock Classifier-free diffusion guidance.
\newblock \emph{arXiv preprint arXiv:2207.12598}, 2022.

\bibitem[Ho et~al.(2020)Ho, Jain, and Abbeel]{ho2020denoising}
Jonathan Ho, Ajay Jain, and Pieter Abbeel.
\newblock Denoising diffusion probabilistic models.
\newblock In \emph{Advances in Neural Information Processing Systems}, volume~33, pages 6840--6851, 2020.

\bibitem[Hoogeboom et~al.(2022)Hoogeboom, Satorras, Vignac, and Welling]{hoogeboom2022equivariant}
Emiel Hoogeboom, V{\'\i}ctor~Garcia Satorras, Cl{\'e}ment Vignac, and Max Welling.
\newblock Equivariant diffusion for molecule generation in {3D}.
\newblock In \emph{International Conference on Machine Learning}, volume 162, pages 8867--8887, 2022.

\bibitem[Hou et~al.(2024)Hou, Zhu, Ren, Bu, Gao, Zhang, and Sun]{hou2024improving}
Xiaoyang Hou, Tian Zhu, Milong Ren, Dongbo Bu, Xin Gao, Chunming Zhang, and Shiwei Sun.
\newblock Improving molecular graph generation with flow matching and optimal transport.
\newblock \emph{arXiv preprint arXiv:2411.05676}, 2024.

\bibitem[Irwin et~al.(2012)Irwin, Sterling, Mysinger, Bolstad, and Coleman]{irwin2012zinc}
John~J Irwin, Teague Sterling, Michael~M Mysinger, Erin~S Bolstad, and Ryan~G Coleman.
\newblock {ZINC}: A free tool to discover chemistry for biology.
\newblock \emph{Journal of Chemical Information and Modeling}, 52\penalty0 (7):\penalty0 1757--1768, 2012.

\bibitem[Jin et~al.(2018)Jin, Barzilay, and Jaakkola]{jin2018junction}
Wengong Jin, Regina Barzilay, and Tommi~S. Jaakkola.
\newblock Junction tree variational autoencoder for molecular graph generation.
\newblock In \emph{International Conference on Machine Learning}, 2018.

\bibitem[Jing et~al.(2022)Jing, Corso, Chang, Barzilay, and Jaakkola]{jing2022torsional}
Bowen Jing, Gabriele Corso, Jeffrey Chang, Regina Barzilay, and Tommi~S. Jaakkola.
\newblock Torsional diffusion for molecular conformer generation.
\newblock In \emph{Advances in Neural Information Processing Systems}, volume~35, 2022.

\bibitem[Karras et~al.(2022)Karras, Aittala, Aila, and Laine]{karras2022elucidating}
Tero Karras, Miika Aittala, Timo Aila, and Samuli Laine.
\newblock Elucidating the design space of diffusion-based generative models.
\newblock In \emph{Advances in Neural Information Processing Systems}, 2022.

\bibitem[Krenn et~al.(2020)Krenn, H{\"a}se, Nigam, Friederich, and Aspuru-Guzik]{krenn2020selfies}
Mario Krenn, Florian H{\"a}se, AkshatKumar Nigam, Pascal Friederich, and Al{\'a}n Aspuru-Guzik.
\newblock {SELFIES}: A robust representation of semantically constrained graphs with an example application in chemistry.
\newblock \emph{Machine Learning: Science and Technology}, 1\penalty0 (4):\penalty0 045024, 2020.

\bibitem[Lee et~al.(2023)Lee, Jo, and Hwang]{lee2023molguidance}
Seul Lee, Jaehyeong Jo, and Sung~Ju Hwang.
\newblock Exploring chemical space with score-based out-of-distribution generation.
\newblock In \emph{International Conference on Machine Learning}, 2023.

\bibitem[Lipman et~al.(2022)Lipman, Chen, Ben-Hamu, Nickel, and Le]{lipman2022flow}
Yaron Lipman, Ricky~TQ Chen, Heli Ben-Hamu, Maximilian Nickel, and Matt Le.
\newblock Flow matching for generative modeling.
\newblock \emph{arXiv preprint arXiv:2210.02747}, 2022.

\bibitem[Liu et~al.(2023)Liu, Gong, and Liu]{liu2023rectified}
Xingchao Liu, Chengyue Gong, and Qiang Liu.
\newblock Flow straight and fast: Learning to generate and transfer data with rectified flow.
\newblock In \emph{International Conference on Learning Representations}, 2023.

\bibitem[Ma et~al.(2024)Ma, Goldstein, Albergo, Boffi, Vanden-Eijnden, and Xie]{ma2024sit}
Nanye Ma, Mark Goldstein, Michael~S. Albergo, Nicholas~M. Boffi, Eric Vanden-Eijnden, and Saining Xie.
\newblock {SiT}: Exploring flow and diffusion-based generative models with scalable interpolant transformers.
\newblock In \emph{ECCV}, 2024.

\bibitem[Peebles and Xie(2023)]{peebles2023dit}
William Peebles and Saining Xie.
\newblock Scalable diffusion models with transformers.
\newblock In \emph{ICCV}, 2023.

\bibitem[Peyr{\'e} and Cuturi(2019)]{peyre2019computational}
Gabriel Peyr{\'e} and Marco Cuturi.
\newblock Computational optimal transport.
\newblock \emph{Foundations and Trends in Machine Learning}, 11\penalty0 (5--6):\penalty0 355--607, 2019.

\bibitem[Pooladian et~al.(2023)Pooladian, Ben-Hamu, Domingo-Enrich, Amos, Lipman, and Chen]{pooladian2023multisample}
Aram-Alexandre Pooladian, Heli Ben-Hamu, Carles Domingo-Enrich, Brandon Amos, Yaron Lipman, and Ricky T.~Q. Chen.
\newblock Multisample flow matching: Straightening flows with minibatch couplings.
\newblock In \emph{International Conference on Machine Learning}, 2023.

\bibitem[Rubinstein(1999)]{rubinstein1999cross}
Reuven~Y Rubinstein.
\newblock The cross-entropy method for combinatorial and continuous optimization.
\newblock \emph{Methodology and Computing in Applied Probability}, 1\penalty0 (2):\penalty0 127--190, 1999.

\bibitem[Segler et~al.(2018)Segler, Kogej, Tyrchan, and Waller]{segler2018generating}
Marwin H.~S. Segler, Thierry Kogej, Christian Tyrchan, and Mark~P. Waller.
\newblock Generating focused molecule libraries for drug discovery with recurrent neural networks.
\newblock \emph{ACS Central Science}, 4\penalty0 (1):\penalty0 120--131, 2018.

\bibitem[Shi et~al.(2020)Shi, Xu, Zhu, Zhang, Zhang, and Tang]{shi2020graphaf}
Chence Shi, Minkai Xu, Zhaocheng Zhu, Weinan Zhang, Ming Zhang, and Jian Tang.
\newblock {GraphAF}: A flow-based autoregressive model for molecular graph generation.
\newblock In \emph{International Conference on Learning Representations}, 2020.

\bibitem[Sohl-Dickstein et~al.(2015)Sohl-Dickstein, Weiss, Maheswaranathan, and Ganguli]{sohldickstein2015deep}
Jascha Sohl-Dickstein, Eric~A. Weiss, Niru Maheswaranathan, and Surya Ganguli.
\newblock Deep unsupervised learning using nonequilibrium thermodynamics.
\newblock In \emph{International Conference on Machine Learning}, 2015.

\bibitem[Song and Ermon(2019)]{song2019generative}
Yang Song and Stefano Ermon.
\newblock Generative modeling by estimating gradients of the data distribution.
\newblock In \emph{Advances in Neural Information Processing Systems}, 2019.

\bibitem[Tong et~al.(2023)Tong, Malkin, Huguet, Zhang, Rector-Brooks, Fatras, Wolf, and Bengio]{Tong2023ImprovingAG}
Alexander Tong, Nikolay Malkin, Guillaume Huguet, Yanlei Zhang, Jarrid Rector-Brooks, Kilian Fatras, Guy Wolf, and Yoshua Bengio.
\newblock Improving and generalizing flow-based generative models with minibatch optimal transport.
\newblock \emph{Trans. Mach. Learn. Res.}, 2024, 2023.
\newblock URL \url{https://api.semanticscholar.org/CorpusID:259847293}.

\bibitem[Vignac et~al.(2023)Vignac, Krawczuk, Siraudin, Wang, Cevher, and Frossard]{vignac2022digress}
Clement Vignac, Igor Krawczuk, Antoine Siraudin, Bohan Wang, Volkan Cevher, and Pascal Frossard.
\newblock {DiGress}: Discrete denoising diffusion for graph generation.
\newblock In \emph{International Conference on Learning Representations}, 2023.

\bibitem[Villani(2003)]{villani2003topics}
C{\'e}dric Villani.
\newblock \emph{Topics in Optimal Transport}.
\newblock American Mathematical Society, 2003.

\bibitem[Weininger(1988)]{weininger1988smiles}
David Weininger.
\newblock {SMILES}, a chemical language and information system: 1. introduction to methodology and encoding rules.
\newblock \emph{Journal of Chemical Information and Computer Sciences}, 28\penalty0 (1):\penalty0 31--36, 1988.

\bibitem[Wildman and Crippen(1999)]{wildman1999prediction}
Scott~A. Wildman and Gordon~M. Crippen.
\newblock Prediction of physicochemical parameters by atomic contributions.
\newblock \emph{Journal of Chemical Information and Computer Sciences}, 39\penalty0 (5):\penalty0 868--873, 1999.

\bibitem[Xiong et~al.(2020)Xiong, Yang, He, et~al.]{xiong2020layer}
Ruibin Xiong, Yunchang Yang, Di~He, et~al.
\newblock On layer normalization in the transformer architecture.
\newblock In \emph{ICML}, 2020.

\bibitem[Xiong et~al.(2025)Xiong, Chen, Li, Zhang, Cai, and Hu]{xiong2025hierarchical}
Yida Xiong, Jiameng Chen, Kun Li, Hongzhi Zhang, Xiantao Cai, and Wenbin Hu.
\newblock Hierarchical bayesian flow networks for molecular graph generation.
\newblock \emph{arXiv preprint arXiv:2510.10211}, 2025.

\bibitem[You et~al.(2018)You, Liu, Ying, Pande, and Leskovec]{you2018graph}
Jiaxuan You, Bowen Liu, Rex Ying, Vijay Pande, and Jure Leskovec.
\newblock Graph convolutional policy network for goal-directed molecular graph generation.
\newblock In \emph{Advances in Neural Information Processing Systems}, 2018.

\bibitem[Zang and Wang(2020)]{zang2020moflow}
Chengxi Zang and Fei Wang.
\newblock {MoFlow}: An invertible flow model for generating molecular graphs.
\newblock \emph{Proceedings of the ACM SIGKDD International Conference on Knowledge Discovery and Data Mining}, 2020.

\bibitem[Zeng et~al.(2025)]{zeng2025propmolflow}
Cheng Zeng et~al.
\newblock {PropMolFlow}: Property-guided molecule generation with geometry-complete flow matching.
\newblock \emph{arXiv preprint arXiv:2505.21469}, 2025.

\bibitem[Zhou et~al.(2019)Zhou, Kearnes, Li, Zare, and Riley]{zhou2019moldqn}
Zhenpeng Zhou, Steven Kearnes, Li~Li, Richard~N. Zare, and Patrick Riley.
\newblock Optimization of molecules via deep reinforcement learning.
\newblock \emph{Scientific Reports}, 9:\penalty0 10752, 2019.

\end{thebibliography}
